\documentclass[journal]{IEEEtran}

\usepackage[english]{babel}

\usepackage{ifpdf}

\usepackage{cite} 
\usepackage{url}
\usepackage{hyperref}

\ifCLASSINFOpdf
	\usepackage[pdftex]{graphicx}
	\graphicspath{{./figures/}}
\else
	\usepackage[dvips]{graphicx}
	\graphicspath{./figures/}
\fi
\usepackage{color}
\usepackage[table]{xcolor} 
\newcommand{\cellpurple}{\cellcolor{pennpurple!15}}
\newcommand{\cellred}{\cellcolor{pennred!15}}
\newcommand{\cellorange}{\cellcolor{pennorange!15}}

\newcommand{\cellgreen}{\cellcolor{penngreen!15}}
\newcommand{\cellblue}{\cellcolor{pennblue!15}}

\usepackage{pgf, tikz, pgfplots}
\usetikzlibrary{shapes, arrows, automata, plotmarks}
\usetikzlibrary{calc,hobby,decorations}

\usepackage[cmex10]{amsmath}
\usepackage{amsfonts, amssymb, amsthm}
\usepackage{mathrsfs}

\usepackage{algorithm,algpseudocode}
	\algnewcommand{\LeftComment}[1]{\Statex \(\triangleright\) #1}

\usepackage{enumerate}
\usepackage{multirow}
\usepackage{rotating}
\usepackage{subcaption}
\usepackage{needspace}

\newcommand{\myparagraph}[1]{\needspace{1\baselineskip}\medskip\noindent {\bf #1}}

\input{auxFiles/mySymbol.sty}
\input{auxFiles/pennColors.sty}

\def\Tr{\mathsf{T}}

\newtheoremstyle{custom}
  {\topsep}   
  {\topsep}   
  {\itshape}  
  {0pt}       
  {\bfseries} 
  {}          
  {5pt plus 1pt minus 1pt} 
  {\thmname{#1}\thmnumber{ #2}:\thmnote{ #3.}} 

\newtheoremstyle{custom_no_it}
  {\topsep}   
  {\topsep}   
  {}          
  {0pt}       
  {\bfseries} 
  {}          
  {5pt plus 1pt minus 1pt} 
  {\thmname{#1}\thmnumber{ #2}:\thmnote{ #3.}} 

\theoremstyle{custom}

\newtheorem{proposition}{\hspace{0pt}\bf Proposition}

\newtheorem{definition}{\hspace{0pt}\bf Definition}
\theoremstyle{custom_no_it}
\newtheorem{remark}{\hspace{0pt}\bf Remark}

\begin{document}

\title{Synthesizing Decentralized Controllers with Graph Neural Networks and Imitation Learning}

\author{Fer\hspace{0.015cm}nando~Gama,~
        Qingbiao~Li,~
        Ekaterina~Tolstaya,~
        Amanda Prorok,~
        and~Alejandro~Ribeiro\vspace{-6mm}
\thanks{The authors are supported by ARL DCIST CRA W911NF-17-2-0181. Q. Li and A. Prorok are also supported by EPSRC grant EP/S015493/1. F. Gama is with the Dept. Elect. Comput. Eng., Rice Univ., E. Tolstaya and A. Ribeiro are with the Dept. Elect. Syst. Eng., Univ. of Pennsylvania., and Q. Li and A. Prorok are with the Dept. Comput. Sci. Technol., Cambridge University. Email: fgama@rice.edu, \{eig,aribeiro\}@seas.upenn.edu, and \{ql295,asp45\}@cam.ac.uk. Preliminary discussions appear in \cite{Tolstaya19-Flocking, Li20-Planning}.
}
}

\markboth{IEEE TRANSACTIONS ON SIGNAL PROCESSING (SUBMITTED)}%
{Synthesizing Decentralized Controllers with Graph Neural Networks and Imitation Learning}

\maketitle

\begin{abstract}
Dynamical systems consisting of a set of autonomous agents face the challenge of having to accomplish a global task, relying only on local information. While centralized controllers are readily available, they face limitations in terms of scalability and implementation, as they do not respect the distributed information structure imposed by the network system of agents. Given the difficulties in finding optimal decentralized controllers, we propose a novel framework using graph neural networks (GNNs) to \emph{learn} these controllers. GNNs are well-suited for the task since they are naturally distributed architectures and exhibit good scalability and transferability properties. We show that GNNs learn appropriate decentralized controllers by means of imitation learning, leverage their permutation invariance properties to successfully scale to larger teams and transfer to unseen scenarios at deployment time. The problems of flocking and multi-agent path planning are explored to illustrate the potential of GNNs in learning decentralized controllers.
\end{abstract}

\begin{IEEEkeywords}
decentralized control, graph neural networks, graph signal processing, flocking, path planning
\end{IEEEkeywords}

\IEEEpeerreviewmaketitle


\section{Introduction} \label{sec:intro}



Network dynamical systems are widespread, spanning applications in multiagent robotics \cite{Nedic10-Consensus, Li20-Planning}, smart grids \cite{Mohsenian10-LoadControl, Owerko20-OPF}, sensor networks \cite{Owerko18-Sensor}, wireless communications \cite{Chiang07-PowerControl, Bechlioulis08-Robust} and traffic control \cite{Li18-DiffusionRNN}. In all of these situations we encounter teams of autonomous agents that sense their local environment and exchange information with nearby agents; which then proceed to control their individual actions in pursuit of a common global goal \cite[Ch. 3]{Ogata02-Control}. This somewhat mismatched specification of attaining a \emph{global} goal with \emph{local} sensing and interaction is the defining characteristic of a network dynamical system.

A possible approach to controlling an autonomous team is to collect information at a designated fusion center that decides on actions that the team executes. Such \emph{centralized} controllers for network systems are common -- see, e.g., \cite{Nedic10-Consensus, Mohsenian10-LoadControl, Chiang07-PowerControl, Bechlioulis08-Robust} -- and their implementation is preferred when feasible. Their advantage is that if communication delays are not significant, the resulting control problem is standard and we can deploy established techniques to design optimal controllers (Sec. \ref{sec:control}). Their disadvantage is that collecting sensor inputs and disseminating control actions, burdens communication networks. As we scale the number of agents in the team, communication delays ensue. This effectively limits applicability of centralized controllers to small teams.

Scalability to teams with large numbers of agents is more feasible with \emph{decentralized} controllers, in which agents decide on their own actions. This approach engenders scalability by design, but results in optimal controllers that are famously difficult to design \cite{Witsenhausen68-Counterexample}. The reason why this happens is that agents that have access to local information, are also agents that have access to different information. This results in the possibility of conflicting actions even for agents that are intent on cooperating (Sec. \ref{sec_decentralized_info}). If optimal decentralized controllers are unavailable, resorting to heuristics is warranted. In this paper we advocate the use of learned heuristics.

The main contribution of this paper is to develop the use of graph neural networks (GNNs) \cite{Gama20-GNNs, Ruiz21-GNNs} to learn decentralized controllers (Sec. \ref{sec:GNN}). More concretely, we model agent interactions with a communication graph and interpret sensing inputs and control actions as signals supported on the nodes of this graph. We then proceed to use graph convolutional neural networks (GCNNs) \cite{Bruna14-DeepSpectralNetworks, Defferrard17-ChebNets, Gama19-Archit} and graph recurrent neural networks (GRNNs) \cite{Seo18-GCRN, Ruiz20-GRNN} to learn maps from sensor inputs to control outputs. The maps are trained with imitation learning \cite{Ross10-ImitationLearning, Ross11-DAGger, Bagnell17-Aggrevated, Hussein17-ImitationLearning, Osa18-ImitationLearning} of expert --ideally but not necessarily optimal-- centralized controllers (Sec. \ref{subsec:imitation}).

The key to explain GNNs and their applicability in decentralized control is the notion of a graph filter \cite{Sandryhaila13-DSPG, Shuman13-SPG, Isufi20-EdgeNets}. Graph filters are generalizations of Euclidean convolutional filters and are formally defined as polynomials on matrix representations of the graph (Sec. \ref{subsec:graphFilters}). Given that the sparsity pattern on these matrix representations matches agent connectivity, graph filters can be implemented in a distributed manner. GCNNs rely on graph filter banks composed with \emph{pointwise} nonlinearities (Sec. \ref{subsec:GCNN}). Given that pointwise operations are local operations, GCNNs also admit distributed implementations. GRNNs also rely on graph filters and pointwise nonlinearities but differ on the incorporation of a time varying hidden state (Sec. \ref{subsec:GRNN}). This makes them more appealing in time varying partially observable problems, but does not affect the admissibility of decentralized implementations.

GNNs admit decentralized implementations by construction and are therefore viable parameterizations to learn decentralized controllers. That they can succeed is a matter of experimental evaluation. Here, we consider flocking (Sec. \ref{sec:flocking}) and path planning (Sec. \ref{sec:pathPlanning}) and deploy GCNNs and GRNNs that are trained to imitate expert centralized controllers. Our first conclusion is that imitation succeeds in both problems over a range of realistic system parameters:

\begin{list}{}{
                 \setlength{\labelwidth}{20pt}
                 \setlength{\leftmargin}{23pt}
                 \setlength{\labelsep}{3pt}
                 \setlength{\itemsep}{5pt}
                 \setlength{\topsep}{5pt}
                 \setlength{\parskip}{0pt}
              }

    \item[\textbf{(C1)}] \textbf{Successful Imitation.} GCNNs and GRNNs are decentralized architectures relying on local information exchanges with neighboring agents. They nevertheless successfully imitate centralized policies that rely on global information.

\end{list}

\vspace{3pt}

\noindent It is important to remark that in imitation learning we rely on expert  centralized controllers for \emph{offline} training of a GNN that we deploy \emph{online} for decentralized control. In this process of transference between systems observed at training to those observed at deployment, networks are likely to change. Asides from the admissibility of decentralized implementations, GNNs have equivariance and stability properties \cite{Gama20-Stability, Ruiz20-GRNN} that make this transference possible. A particular change of interest are permutations that result out of agent relabeling (Sec. \ref{sec:permutation}). Our second conclusion is to prove that:

\begin{list}{}{
                 \setlength{\labelwidth}{20pt}
                 \setlength{\leftmargin}{23pt}
                 \setlength{\labelsep}{3pt}
                 \setlength{\itemsep}{5pt}
                 \setlength{\topsep}{5pt}
                 \setlength{\parskip}{0pt}
              }

    \item[\textbf{(C2)}] \textbf{Permutation Invariance.} Optimal GNNs are invariant to permutations of the agents (Proposition \ref{thm:permutationInvariance}) if state observations are properly crafted (Remarks \ref{rmk_sys_invariance_no} and \ref{rmk_sys_invariance_yes}).

\end{list}

\vspace{3pt}

\noindent Conclusion (C2) is a direct consequence of the permutation equivariance of GCNNs \cite{Gama20-Stability} and GRNNs \cite{Ruiz20-GRNN}. It guarantees that optimal GNNs are invariant to relabeling of the agents of the team. This is important because most problems in decentralized systems satisfy this. This conclusion is also relevant for problems that are not invariant to relabeling. GNNs are not appropriate learning parameterizations for them.

When moving from training to execution, network realizations are unlikely to be exact permutations of each other. Rather, we expect to see \emph{online} networks that are close to permutations of each other. Our flocking (Sec. \ref{sec:flocking}) and path planning (Sec. \ref{sec:pathPlanning})  experiments demonstrate a third important conclusion:

\begin{list}{}{
                 \setlength{\labelwidth}{20pt}
                 \setlength{\leftmargin}{23pt}
                 \setlength{\labelsep}{3pt}
                 \setlength{\itemsep}{5pt}
                 \setlength{\topsep}{5pt}
                 \setlength{\parskip}{0pt}
              }

    \item[\textbf{(C3)}] \textbf{Transference.} GCNNs and GRNNs can be transferred from systems drawn at \emph{offline} training into systems observed during \emph{online} execution. This is true even if offline and online systems are not exact permutations of each other [cf. (C2)].

\end{list}

\vspace{3pt}

\noindent While conclusion (C3) is not unexpected due to the stability property of GNNs \cite{Gama20-GNNs}, the experiments in Secs.~\ref{sec:flocking}~and~\ref{sec:pathPlanning} demonstrate that this result holds over a range of reasonable system parameters. We also show that GNNs can run over time-varying graphs (Sec.~\ref{sec:timeVarying}), further supporting (C3) for this practically relevant scenario.

Finally, GNNs are also scalable to systems larger than those they were trained on:

\noindent

\begin{list}{}{
                 \setlength{\labelwidth}{20pt}
                 \setlength{\leftmargin}{23pt}
                 \setlength{\labelsep}{3pt}
                 \setlength{\itemsep}{5pt}
                 \setlength{\topsep}{5pt}
                 \setlength{\parskip}{0pt}
              }
    \item[\textbf{(C4)}] \textbf{Transference at Scale.}  GCNNs and GRNNs can be transferred to systems with larger numbers of nodes.
\end{list}

\vspace{3pt}

\noindent This key property of GNNs allows for imitation learning of computationally intensive experts. GCNNs and GRNNs are trained offline on small systems where the expert is feasible, and then scaled online to larger systems where an expert is not computationally available. Conclusion (C4) is supported by the known transferability properties of GNNs \cite{Ruiz20-Transferability}.

Overall, the main objective is to illustrate how the use of GNNs in combination with imitation learning can lead to a powerful framework for learning decentralized controllers that scale up to large multiagent systems. We describe imitation learning, discuss the properties of GNNs, and develop insights on why these properties play a key role in learning powerful decentralized controllers. We place particular emphasis in the importance of converting a permutation equivariant controller into a permutation invariant formulation that allows for scalability. We then illustrate the success of the proposed framework with two, seemingly dissimilar, examples drawn from robotics: flocking \cite{Tolstaya19-Flocking} and path planning \cite{Li20-Planning}. We recast these problems in a permutation invariant manner and include more insightful numerical experiments. We consider a wider range of architectures and useful practical scenarios. The new numerical experiments better showcase the properties of permutation invariance, transferability and scalability. In essence, we present a unified framework for synthesizing decentralized controllers with graph neural networks and imitation learning. This same problem formulation is useful in describing many other problems, both in the realm of robotics and beyond, providing an efficient and straightforward methodology for learning scalable solutions from computationally inefficient experts.

The proposed framework learns decentralized controllers by using GNNs and shows promising results in relevant robotics problems. These results are backed theoretically by properties GNNs have that make them sensible parametrization choices. The proposed use of GNNs for learning decentralized controllers opens interesting directions for future research, including stability analysis, exploitation of the dynamic models and exploration of alternative learning frameworks (Sec.~\ref{sec:conclusions}). Conceptual insights and practical tools developed in graph signal processing can potentially make important contributions to learning decentralized controllers and understanding their behavior.


\section{Optimal Decentralized Control} \label{sec:control}



Consider a team of $N$ agents $\ccalV = \{1,\ldots,N\}$. At time $t \in \{0,1,2,\ldots\}$, each agent $i \in \ccalV$ is described by a state $\bbx_{i}(t) \in \reals^{F}$ and is equipped with the faculty of autonomously deciding on an action $\bbu_{i}(t) \in \reals^{G}$. States and actions of all the agents in the team are collected in the state matrix $\bbX(t) \in \reals^{N \times F}$ and the action matrix $\bbU(t) \in \reals^{N \times G}$ in which individual rows correspond to the state and action of each agent,
\begin{equation} \label{eq:stateMatrix}
    \bbX(t) = 
    \begin{bmatrix} 
        \bbx_{1}(t)^{\Tr} \\ 
        \vdots \\ 
        \bbx_{N}(t)^{\Tr}
    \end{bmatrix} ~, 
    \qquad
    \bbU(t) = 
    \begin{bmatrix} 
        \bbu_{1}(t)^{\Tr} \\ 
        \vdots \\ 
        \bbu_{N}(t)^{\Tr}
    \end{bmatrix} ~.
\end{equation}
The collective effect of the actions of \emph{all} agents is to drive the evolution of the state of the \emph{team} as determined by a Markov model $\mbP$. Thus, the state of the system at time $t+1$ is drawn from the distribution
\begin{equation} \label{eq:dynamicModel}
    \bbX(t+1) \sim \mbP \Big( \bbX \given \bbX(t), \bbU(t) \Big).
\end{equation}
When the system is in state $\bbX(t)$ the execution of action $\bbU(t)$ incurs a cost $c(\bbX(t), \bbU(t))$. Given this specification it is customary to search for a possibly randomized policy that minimizes the expected cost over the system's trajectory. Formally, we define the control policy $\Pi_{\text{c}}$ as a conditional distribution from which we draw actions as $\bbU(t) = \Pi_{\text{c}}(\bbU \given \bbX(t))$. Introducing a discount factor $\gamma$, the optimal policy is defined as the one that minimizes the expected cost
\begin{equation} \label{eq:generalObjective}
    \Pi_{\text{c}}^*
      = \argmin_{\Pi_{\text{c}}}\mbE
           \bigg[ 
               \sum_{t=0}^{\infty}\!\gamma^t  
                  c\Big(\bbX(t), \Pi_{\text{c}}(\bbX(t)) \Big)
                     \bigg]
      = \argmin_{\Pi_{\text{c}}} 
            J (\Pi_{\text{c}}),
\end{equation}
where we defined $J(\Pi_{\text{c}})=\mbE[ \sum_{t=0}^{\infty} \gamma^t c(\bbX(t), \Pi_{\text{c}}(\bbX(t))) ]$ in the second equality to represent the average long term discounted cost associated with policy $\Pi_{\text{c}}$. Depending on the dynamical model $\mbP$ finding the optimal policy $\Pi^*_{\text{c}}$ has varying degrees of complexity. In this paper we assume that $\Pi^*_{\text{c}}$ has been computed and is available to facilitate the design of a \emph{decentralized} controller; see Remark \ref{rmk_optimal_centralized_policy}.

\subsection{Decentralized Information Structure}\label{sec_decentralized_info}

Decentralized control problems are characterized by limited access to agents' states and actions. To be more specific introduce a local connectivity pattern specified by a symmetric graph $\ccalG = (\ccalV, \ccalE)$ in which the vertices $\ccalV$ represent nodes and the edges $\ccalE \subseteq \ccalV \times \ccalV$ represent communication links. Thus, when $(i,j)\in\ccalE$ agents $i$ and $j$ share a communication link and are able to exchange information -- see Secs. \ref{sec:flocking} and \ref{sec:pathPlanning} for concrete communication models. 
The set $\ccalN_{i} = \{j \in \ccalV: (j,i) \in \ccalE\}$ of nodes $j$ that share an edge with $i$ is called the neighborhood of $i$. In decentralized control multiple hop neighbors are also of interest. These are recursively defined as $\ccalN_i^0=\{i\}$ and $\ccalN_i^{k+1}=\ccalN_i^k\cup\ccalN_i$. Thus, the $k$-hop neighborhood  $\ccalN_i^k$ is made up of nodes $j$ that can reach $i$ in no more than $k$ edge transitions. The $k$-hop neighborhoods determine the information that is available to node $i$ at time $t$. Indeed, assuming that communication between nodes follows the same clock of sensing and control, the information that is available at node $i$ at time $t$ is given by the set
%
%
\begin{equation} \label{eq:neighborHistory}
\ccalX_{i}(t) = \bigcup_{k=0}^{\infty} \Big\{ \bbx_{j}(t-k) \ , \ j  \in \ccalN_{i}^{k}(t)\Big\}
\end{equation}
That is, agent $i$ has access to its \emph{current} state $\bbx_i(t)$ but the information of other agents makes it to agent $i$ with some delay. The states  $\bbx_j(t-1)$ of 1-hop neighbors $j \in \ccalN_{i}^1$ at time $t-1$ can be known to node $i$, but the current state need not. In general, agent $i$ can know states $\bbx_j(t-k)$ of $k$-hop neighbors $j \in \ccalN_{i}^{k}$ with $k$-time-unit delays but does not have access to more current information except for those $k$-hop neighbors that are also $(k-1)$-hop neighbors. 

A decentralized controller is one that respects the local information structure defined by \eqref{eq:neighborHistory}. Formally, define local policies $\pi_i$ as those that map information $\ccalX_{i}(t)$ to actions $\bbu_i(t) = \pi_i (\ccalX_{i}(t))$. A decentralized control policy $\Pi$ is made up of the joint execution of these decentralized policies. Upon defining $\ccalX(t)=[\ccalX_1(t),\ldots,\ccalX_N(t)]$ we write \emph{decentralized} control policies as maps
\begin{equation} \label{eqn_local_policies}
    \bbU(t) = \Pi\Big(\ccalX(t)\Big), 
       \text{~with~}  
          \bbu_i(t) = \pi_i \Big(\ccalX_{i}(t)\Big) .
\end{equation}
The optimal decentralized control problem is the equivalent of \eqref{eq:generalObjective} when policies are restricted to the form in \eqref{eqn_local_policies}
\begin{equation} \label{eq:optimalDecentralizedController}
    \Pi^*
      = \argmin_{\Pi}\mbE
           \bigg[ 
               \sum_{t=0}^{\infty}\!\gamma^t  
                  c\Big(\bbX(t), \Pi\big(\ccalX(t)\big) \Big)
                     \bigg]
      = \argmin_{\Pi} 
            J (\Pi),
\end{equation}
As is the case of \eqref{eq:generalObjective} the difficulty of finding $\Pi^*$ depends on the dynamical model $\mbP$. However, a cursory inspection of \eqref{eq:optimalDecentralizedController} suggests that, in general, decentralized control is much more difficult than centralized control. This is because the  information structure defined by \eqref{eq:neighborHistory} involves trajectory histories. The centralized controller is Markov and depends on the current state $\bbX(t)$. The decentralized controller requires storage and processing of past information. It is, in fact, well known that optimal decentralized control is intractable even in the case of linear quadratic regulators that have elementary centralized solutions \cite{Witsenhausen68-Counterexample}. 

This increase in complexity of decentralized controllers relative to centralized controllers prompts the development of learning techniques for decentralized control. The approach advanced by this paper relies on the use of GNNs (Secs. \ref{subsec:graphFilters} through \ref{subsec:GRNN}) and imitation learning (Sec. \ref{subsec:imitation}). The use of GNNs is justified because they respect the local information structure of decentralized control by design. The use of imitation learning is motivated by the relative simplicity of designing centralized controllers relative to the design of decentralized controllers.

\begin{remark}[Centralized controllers]\label{rmk_optimal_centralized_policy}\normalfont We have defined the optimal centralized controller in \eqref{eq:generalObjective} and the optimal decentralized controller in \eqref{eq:optimalDecentralizedController} to highlight their differences. The learning methodology that we develop in Sec. \ref{sec:GNN} is predicated on the availability of \emph{an expert} (an potentially centralized) controller to imitate, which need not be \emph{the} optimal centralized controller in \eqref{eq:generalObjective}. This is important because even in cases where optimal centralized controllers are difficult to find, centralized control heuristics that outperform decentralized heuristics are often easier to devise; see Sec. \ref{sec:pathPlanning} for an example.
\end{remark}






\section{Graph Neural Networks} \label{sec:GNN}



Optimal decentralized controllers are difficult to find due to the constraint in \eqref{eq:optimalDecentralizedController} that forces the controller to be a function only of neighboring information. To address this, we propose to parametrize $\bbU(t) = \sfPhi(\ccalX(t))$ with a function $\sfPhi$ that is naturally distributed, so that the constraint is always satisfied. In particular, we restrict our attention to controllers that can be computed by means of a graph neural network (GNN). While the obtained solutions are likely to be suboptimal, they are easy to train, can be computed in a distributed manner and exhibit several properties that make them a reasonable choice of parametrization, besides playing a key role in the practical success observed in Secs.~\ref{sec:flocking}~and~\ref{sec:pathPlanning}. In short, we restrict the functional optimization problem \eqref{eq:optimalDecentralizedController} to be a parameter optimization problem over the space of GNNs, which are nonlinear, local and distributed.

We introduce the framework of graph signal processing (GSP) and the fundamental concept of graph filters, allowing for a straightforward mathematical description of decentralized problems (Sec.~\ref{subsec:graphFilters}). Graph convolutional neural networks (GCNNs) build upon graph filters by including pointwise nonlinearities to conform a layer, and cascading several layers as means of increasing their representation power (Sec.~\ref{subsec:GCNN}). Graph recurrent neural networks (GRNNs) further incorporate the time dimension by learning a sequence of \emph{hidden} states that keep track of the relevant information in the system evolution (Sec.~\ref{subsec:GRNN}).
Finally, we introduce imitation learning as a way of training these architectures to learn useful decentralized controllers (Sec.~\ref{subsec:imitation}).


\subsection{Graph filters} \label{subsec:graphFilters}

The distributed nature of the decentralized controllers comes from the fact that each agent can only communicate with other nearby agents. As discussed in Sec.~\ref{sec:control}, we describe this communication network by means of a graph $\ccalG = (\ccalV, \ccalE)$, where $\ccalV$ is the set of nodes (agents) and $\ccalE$ is the set of edges (communication links). A decentralized controller, thus, is only allowed to process information that propagates through this graph. We use graph signal processing (GSP) \cite{Sandryhaila13-DSPG, Shuman13-SPG, Ortega18-GSP} as the appropriate framework to learn such controllers.

A \emph{graph signal} $\sfx: \ccalV \to \reals$ is defined as a function that assigns a scalar value to each node and can be conveniently described by a vector $\bbx \in \reals^{N}$ where $[\bbx]_{i} = x_{i} = \sfx(i)$ is the signal value assigned to node $i$. We can extend this concept to describe the collection of states of all agents in the team $\bbX \in \reals^{N \times F}$ [cf. \eqref{eq:stateMatrix}], by assigning a $F$-dimensional vector to each agent $\sfX: \ccalV \to \reals^{F}$ so that $\sfX(i) = \bbx_{i} \in \reals^{F}$, corresponding to the rows of $\bbX$. For ease of exposition, we refer to $\bbX$ as a graph signal as well, even though, technically, $\bbX$ is a collection of $F$ graph signals.

A graph signal $\sfX : \ccalV \to \reals^{F}$ contains information only about the nodes in the graph $\ccalG$, but bears no relation with the edge set $\ccalE$ that determines the topology of the graph. This information is captured in another matrix $\bbS \in \reals^{N \times N}$ which is a description of the \emph{support} of the graph and satisfies $[\bbS]_{ij} = s_{ij}= 0$ whenever $i \neq j$ or $(j, i) \notin \ccalE$. Note that the support matrix $\bbS$ respects the sparsity of the graph since there is a zero entry whenever agents $i$ and $j$ do not share a communication link between each other. Examples of the support matrix typically used in the literature include the adjacency matrix \cite{Sandryhaila13-DSPG}, the Laplacian matrix \cite{Shuman13-SPG}, the Markov matrix \cite{Heimowitz17-MarkovGSP}, as well as normalized counterparts \cite{Gama20-Sketching}.

The fact that the support matrix $\bbS$ respects the sparsity of the graph is the key property that allows it to conveniently describe distributed operations in a straightforward manner. To see this, consider the linear operation $\bbS \bbX$ that results from applying the support matrix $\bbS$ to a graph signal $\bbX$. The output is another graph signal whose $(i,f)$th entry is computed as
\begin{equation} \label{eq:graphShift}
    [\bbS \bbX]_{if} = \sum_{i=1}^{N} [\bbS]_{ij} [\bbX]_{jf} = \sum_{j \in \ccalN_{i} \cup i} s_{ij} x_{jf}.
\end{equation}
The operation $\bbS \bbX$ acts as a diffusion operator that \emph{shifts} the information contained in the state $\bbX$ across the graph, thus oftentimes $\bbS$ receives the name of \emph{graph shift operator} (GSO).

The second equality in \eqref{eq:graphShift} holds due to the sparsity pattern of the support matrix $\bbS$, that is, there are nonzero entries only when $(j, i) \in \ccalE$. This implies that the graph signal that results from the linear operation $\bbS\bbX$ requires local information that is provided by neighboring nodes only. Likewise, each node can separately compute this output without knowledge of the rest of the graph, only knowing their one-hop neighbors. In essence, $\bbS \bbX$ is a \emph{local} and \emph{distributed} operator, and thus serves as the building block for learning \emph{decentralized} controllers.

Consider graph filters $\sfH: \reals^{N \times F} \to \reals^{N \times G}$ that map a graph signal $\bbX \in \reals^{N \times F}$ into another one $\bbY \in \reals^{N \times G}$ by means of a polynomial in the support matrix
\begin{equation} \label{eq:graphFilter}
    \bbY = \sum_{k=0}^{K} \bbS^{k} \bbX \bbH_{k} = \sfH(\bbX;\bbS, \ccalH)
\end{equation}
where $\ccalH = \{\bbH_{k} \in \reals^{F \times G} \ , \ k = 0 ,\ldots, K\}$ is the set of \emph{filter coefficients}, \emph{filter weights} or \emph{filter taps}.
These filters \eqref{eq:graphFilter} are local and distributed mappings between graph signals. They are local since they only require communication exchanges with one-hop neighbors
, and they are distributed since each node can compute the output of the filter separately, by aggregating the information shared $K$ times by their one-hop neighbors and then weighing this information by the filter taps contained in $\bbH_{k}$. We note that there is no need for the nodes to know $\bbS$ (or the entire graph) at implementation time, they only need to know their one-hop neighbors $\ccalN_{i}$, and the corresponding filter taps $\ccalH$.

\begin{remark}[Choice of $\bbS$] \label{rmk:choiceS} \normalfont
    The choice of support matrix $\bbS$ affects the performance of the resulting controller, and thus different alternatives exhibit different desirable properties \cite{Mateos19-TopologyInference}. However, no single choice has been shown to outperform others in a wide range of problems, and thus we keep a generic support matrix $\bbS$ that acts as a stand-in for any matrix that respects the sparsity of the graph. The results we obtain in this paper hold for any problem-specific choice of support matrix.
\end{remark}

\subsection{Graph convolutional neural networks} \label{subsec:GCNN}

Graph filters \eqref{eq:graphFilter} are linear maps between graph signals, and as such, are only able to capture linear relationships if used to design decentralized controllers. However, it is known that these controllers are typically nonlinear \cite{Witsenhausen68-Counterexample}. Therefore, we need a map that is capable of capturing nonlinear relationships if we are to design successful decentralized controllers.

Graph convolutional neural networks (GCNNs) are a cascade of operational blocks or \emph{layers}, each of which applies a graph filter \eqref{eq:graphFilter} followed by a pointwise nonlinearity \cite{Bruna14-DeepSpectralNetworks, Defferrard17-ChebNets, Gama19-Archit}. Let $\bbX_{\ell} \in \reals^{N \times F_{\ell}}$ be the signal obtained at the output of layer $\ell$. This output is computed as
\begin{equation} \label{eq:GCNN}
    \bbX_{\ell} = \sigma \big( \sfH_{\ell} (\bbX_{\ell-1}; \bbS, \ccalH_{\ell}) \big)
\end{equation}
where $\sfH_{\ell}: \reals^{N \times F_{\ell-1}} \to \reals^{N \times F_{\ell}}$ is a graph filter of the form \eqref{eq:graphFilter} described by the set of filter coefficients $\ccalH_{\ell} = \{\bbH_{\ell k} \in \reals^{F_{\ell-1} \times F_{\ell}}, k = 0,\ldots,K_{\ell}\}$, and $\sigma:\reals \to \reals$ is a nonlinear function that, in an abuse of notation, is used in \eqref{eq:GCNN} to denote its element-wise application to the entries of the output of the graph filter. The input $\bbX_{0}$ to the first layer is the data, $\bbX_{0}=\bbX$, and we collect the output at the last layer
\begin{equation} \label{eq:GCNNout}
    \sfPhi(\bbX ; \bbS, \ccalH) = \bbX_{L}
\end{equation}
to be the output of the GCNN $\sfPhi$, and where $\ccalH = \{\ccalH_{\ell}, \ell=1,\ldots,L\}$ is the set of all filter taps used in all layers. The number of layers $L$, the number of features per layer $F_{\ell}$, the order of each graph filter $K_{\ell}$, and the nonlinear function $\sigma$ are all design choices, and are normally called \emph{hyperparameters} (in contrast to the filter taps, that often receive the name of parameters). In this respect, traditional techniques of hyperparameter selection hold for GCNNs as well \cite{Bergstra12-RandomSearch}.

GCNNs extend graph filters by including pontwise nonlinearities and cascading several operational blocks. Thus, they retain many of the properties that graph filters exhibit, namely their local and distributed nature. Therefore, GCNNs are suitable choices for learning nonlinear decentralized controllers.

\subsection{Graph recurrent neural networks} \label{subsec:GRNN}

Graph recurrent neural networks (GRNNs) are nonlinear information processing architectures better suited for handling graph processes (i.e. sequences of graph signals) \cite{Perraudin17-Stationarity, Marques17-Stationarity, Gama19-GLLN} since they exploit both the graph and time structure of data \cite{Seo18-GCRN, Ruiz20-GRNN}. A GRNN learns to extract information from the sequence $\{\bbX(t)\}$ in the form of a sequence of \emph{hidden states} $\{\bbZ(t)\}$. Each hidden state is a graph signal $\bbZ(t) \in \reals^{N \times H}$ and is learned from the input process by using a nonlinear map that takes the current data point $\bbX(t)$ and the previous hidden state $\bbZ(t-1)$ as inputs, and outputs the updated hidden state $\bbZ(t)$. This map is parametrized  as
\begin{equation} \label{eq:GRNNhidden}
    \bbZ(t) = \sigma \Big( \sfA(\bbX(t); \bbS, \ccalA) + \sfB(\bbZ(t-1); \bbS, \ccalB \Big)
\end{equation}
where $\sfA: \reals^{N \times F} \to \reals^{N \times H}$ and $\sfB: \reals^{N \times H} \to \reals^{N \times H}$ are graph filters \eqref{eq:graphFilter} characterized by the set of filter taps $\ccalA = \{\bbA_{k} \in \reals^{F \times H}, k = 0,\ldots,K\}$ and $\ccalB = \{\bbB_{k} \in \reals^{H \times H}, k = 0,\ldots,K\}$, respectively. The function $\sigma:\reals \to \reals$ is a pointwise nonlinearity that is applied elementwise to the graph signal that results from adding up the two filter outputs.

A controller can be learned from the hidden state by means of another nonlinear map
\begin{equation} \label{eq:GRNNout}
    \bbU(t) = \rho \Big( \sfC(\bbZ(t); \bbS, \ccalC)\Big)
\end{equation}
where $\sfC: \reals^{N \times H} \to \reals^{N \times G}$ is a graph filter \eqref{eq:graphFilter} characterized by the filter taps $\ccalC = \{\bbC_{k} \in \reals^{H \times G}, k= 0,\ldots,K_{o}\}$, and where $\rho:\reals \to \reals$ is a pointwise nonlinearity that is applied to each entry of the output of the filter. In the case of a GRNN \eqref{eq:GRNNhidden}-\eqref{eq:GRNNout}, the design hyperparameters are the size of the hidden state $H$, the order of the filters $K$ in \eqref{eq:GRNNhidden} and $K_{o}$ in \eqref{eq:GRNNout}, and the choice of nonlinearities $\sigma$ and $\rho$. 

GRNNs \eqref{eq:GRNNhidden}-\eqref{eq:GRNNout} are local and distributed architectures, since they are the composition of graph filters \eqref{eq:graphFilter}, that exhibit these properties, and pointwise nonlinearities, that do not alter them. Oftentimes, we choose $K_{o} = 0$ so that the output controller \eqref{eq:GRNNout} is obtained by combining the hidden state values at each node, saving on communication cost and avoiding delays [cf. \eqref{eq:graphFilter}]; in other words, the communication with one-hop neighbors is carried out by \eqref{eq:GRNNhidden} and used to learn an appropriate hidden state. We also note that, potentially, we can also leverage the hidden state to learn specific messages to transmit that could be different from the signals $\bbX(t)$. In any case, GRNNs are also suitable choices for learning nonlinear decentralized controllers. We note that the number of parameters in $\ccalA$, $\ccalB$, and $\ccalC$ is independent of both $N$ the size of the graph and the length of the sequence. Thus, GRNNs can seamlessly adapt to arbitrarily long sequences, albeit they might require gating mechanisms to improve training \cite{Ruiz20-GRNN}.

The consideration of GRNN architectures leads to an increased representational power by keeping track of the evolution of the system through hidden state [cf. \eqref{eq:GRNNhidden}]. In this way, GRNNs decouple the mapping from the input to the output into one from the input to the hidden state and another one from the hidden state to the output. This provides a more powerful framework capable of learning a wider variety of representations. Potentially, this hidden state can be used to map, not only the output, but also the specific messages to transmit. That is, GRNNs can be used to learn specific messages to communicate to neighbors based on the current task. Furthermore, the use of gating that is typical to GRNNs can be employed to determine when and where to communicate between agents, potentially saving on communication costs.


\subsection{Imitation learning} \label{subsec:imitation}

Learning a decentralized controller based on a graph filter \eqref{eq:graphFilter}, a GCNN \eqref{eq:GCNN}-\eqref{eq:GCNNout} or a GRNN \eqref{eq:GRNNhidden}-\eqref{eq:GRNNout} entails a training procedure to find appropriate parameters. This could mean the graph filter taps $\ccalH$, the GCNN filter tensor $\ccalH$, or the GRNN filter taps $\ccalA$, $\ccalB$ and $\ccalC$. For simplicity of exposition we describe training of a GCNN. Training graph filters or GRNNs are ready extensions.

We propose the use of imitation learning \cite{Ross10-ImitationLearning, Ross11-DAGger, Bagnell17-Aggrevated, Hussein17-ImitationLearning, Osa18-ImitationLearning} whereby the filter tensor $\ccalH$ is chosen as the one that can most closely imitate the expert controller $\Pi^*_{\text{c}}$ defined in \eqref{eq:optimalDecentralizedController}. Formally, we seek a filter tensor for which the norm of the difference between the GCNN actions $\sfPhi(\bbX; \bbS, \ccalH)$ and the expert controller actions $\Pi^*_{\text{c}}(\bbX)$ is minimized on expectation over the probability distribution of the state $\bbX$ under the expert policy,
\begin{equation} \label{eq:imitationLearning}
   \ccalH^{\ast}
      = \argmin_{\ccalH}\mbE
           \bigg[ \,
              \Big \| \,
                 \sfPhi\big(\bbX; \bbS, \ccalH\big)
                    -  \Pi^*_{\text{c}}\big(\bbX \big) \,
                       \Big \| \,
                          \bigg]
\end{equation}
where $\| \cdot \|$ stands for the norm of a graph signal $\bbX \in \reals^{N \times F}$ given by $\|\bbX\| = \sum_{f=1}^{F} \| \bbx_{f}\|_{2}$ \cite{Gama22-DistributedLQR} (for $\|\cdot\|_{2}$ the Euclidean norm of vectors), which is equivalent to the $L_{2,1}$ matrix norm.
Observe that in \eqref{eq:imitationLearning} the probability distribution of the state $\bbX$ depends on the control policy. We assume that state realizations are drawn from a system that is controlled by $\Pi^*_{\text{c}}$. In practice, we simulate realizations of the system controlled by $\Pi^*_{\text{c}}$ to generate system trajectories $\bbX(t)$. This creates a training set $\ccalT$ of trajectories that we use to approximate the expected cost in \eqref{eq:imitationLearning},
\begin{equation} \label{eq:imitationLearning_sum}
   \hat{\ccalH}^{\ast}
      = \argmin_{\ccalH} \sum_{\bbX(t)\in\ccalT}
           \bigg[ \,
              \Big \| \,
                 \sfPhi\big(\bbX(t); \bbS, \ccalH\big)
                    -  \Pi^*_{\text{c}}\big(\bbX(t) \big) \,
                       \Big \| \,
                          \bigg].
\end{equation}
The tensors in \eqref{eq:imitationLearning} and \eqref{eq:imitationLearning_sum} are close for a sufficiently large training set. We note that the optimization problem \eqref{eq:imitationLearning_sum} is nonconvex, and thus the use of any algorithm based on stochastic gradient descent (SGD) is not guaranteed to converge to the global optimum. In any case, as is common practice in machine learning, the objective of using SGD on \eqref{eq:imitationLearning_sum} is not to find a global optimum, but to find suitable parameters $\ccalH$ that would lead to proper generalization to the test set~\cite[Ch. 8]{Goodfellow16-DeepLearning} --a feat achieved as evidenced in Secs.~\ref{sec:flocking}~and~\ref{sec:pathPlanning}.

Imitation learning is a general framework that can be used beyond GNNs. The main concept is that one has access to some expert controller that is readily available, but with undesirable characteristics. One can then design a family of controllers with desirable characteristics and leverage imitation learning to choose the best controller among that family. In this particular case, the undesirable characteristics of the expert controller are either its centrality or its computationally intractability (i.e. many decentralized controllers are NP-hard and thus do not scale for large teams \cite{Prorok21-HolyGrail}), while the desirable characteristics of GNNs are its computational tractability (decentralized nature and easy to compute) and, by the permutation invariance property (Sec.~\ref{sec:permutation}), transferability and scalability.

Implementation of \eqref{eq:imitationLearning_sum} requires access to the optimal expert controller $\Pi^*_{\text{c}}$. We emphasize that this is needed during \emph{offline training} but is not needed during \emph{online execution.} At execution time, agents observe their states and exchange messages with neighbors to evaluate the outputs of the graph filters that are defined by the filter tensor $\ccalH$. We say that the GCNN has been \emph{transferred} from the \emph{offline} system to the \emph{online} system. For this transference to work well, offline and online systems must be similar. In the following section, we show that GNNs have perfect transference across systems that are permutations of each other. As we will see, this is a simple yet necessary property for successful use of imitation policies.



%
\begin{remark}[Closed Loop Stability]
While imitating a closed-loop stable expert controller by means of \eqref{eq:imitationLearning_sum} does not guarantee that the resulting decentralized controller is stable, experimental results in Secs. \ref{sec:flocking} and \ref{sec:pathPlanning} show that they can be.
Theoretical closed loop stability guarantees are beyond the scope of this paper but some stability guarantees in related settings are known. We refer the reader to \cite{Matni21-StabilityIL, Arcak2021-ILstability} for results on stability, generalizability and robustness of policies trained by imitation learning. For stability guarantees of GNN-based controllers we refer the reader to \cite{Gama21-LQRGNN, Gama22-DistributedLQR}, which considers the particular case of linear dynamics and quadratic costs.
\end{remark}

%
\begin{remark}[Decentralized Training]
We emphasize that in \eqref{eq:imitationLearning_sum} the training procedure is assumed to be centralized. This is not necessarily a drawback because we train offline prior to deployment of the learned policy in a decentralized system. Decentralized training is of interest to adapt policies to system realizations observed during online execution. We point out, however, that decentralized training is incompatible with imitation of a centralized expert policy. A possible approach is to combine a policy that is trained offline through centralized imitation learning with a policy that is trained online with decentralized optimization; see \cite{Gao20-WideDeep}.
\end{remark}

%
\begin{remark}[Reinforcement Learning]
Instead of imitation learning, the use of reinforcement learning (RL) \cite{Sutton18-RL} is also possible \cite{Khan19-GraphPolicyGradients}. Relative to RL, imitation learning is more sample efficient \cite{Ross10-ImitationLearning, Ross11-DAGger, Bagnell17-Aggrevated} and has been, for that reason, extensively used in situations where an expert controller is available \cite{Hussein17-ImitationLearning, Osa18-ImitationLearning, Hertneck17-ImitationMPC}. The drawback of imitation learning is the need to have access to this expert controller. In our case, this entails access to a centralized controller which may require the use of heuristics (cf. Remark \ref{rmk_optimal_centralized_policy}). If available heuristics are deemed inadequate, an optimal centralized controller may itself imply the use of RL to solve \eqref{eq:generalObjective}. In this case, the use of RL to directly train a GNN is a better alternative \cite{Khan19-GraphPolicyGradients}.
\end{remark}


\section{Permutation Invariance and Transferability} \label{sec:permutation}


Implicit in the use of imitation learning is the assumption that the system realizations that are observed during \emph{offline training} are similar to the system realizations that are observed during \emph{online execution.} A significant obstacle to satisfy this requirement are agent labels. Changing agent labels does not change the system, but it changes the description of the state. Since control actions produced by a GNN are functions of the state, it could be that we execute different actions for the same system. In this sense, the permutation equivariance of GNNs \cite[Prop.~1]{Gama20-Stability} implies that a reordering of the input will result in the same reordering of the output. To obtain a scalable controller, we actually need guarantees that the problem to solve is permutation invariant. In this section, we prove that the imitation learning leads to a permutation invariant problem when the input is properly chosen.

Begin by defining a permutation matrix of size $N$ as a binary matrix with exactly one nonzero entry in each row and column,
\begin{equation} \label{eq:permutationSet}
   \bbP \in \{0,1\}^{N \times N},
   \text{\quad with~}
   \bbP \bbone = \bbone,~
   \bbP^{\Tr} \bbone = \bbone.
\end{equation}
Premultiplication of a vector by the permutation matrix $\bbP^{\Tr}$ reorders its entries. Consequently, the product $\tbX=\bbP^{\Tr}\bbX$ is a permutation of the system's state and the product $\tbU=\bbP^{\Tr}\bbU$ is a permutation of the system's action. Given that permutation matrices are orthogonal, this change of variables does not change the transition probability kernel. We therefore have
\begin{equation} \label{eqn_dynamic_model_invariance}
    \mbP \Big( \tbX \given \tbX(t), \tbU(t) \Big)
    ~=~
    \mbP \Big( \bbX \given \bbX(t), \bbU(t) \Big).
\end{equation}
It follows from \eqref{eqn_dynamic_model_invariance} that the evolution of the system is the same irrespectively of how entries are labeled -- as it should be. To attain equivalent systems we have to further assume that costs are invariant to permutations as we detail next.


\begin{definition}[Permutation Invariant Systems]\label{def_sys_invariance} Consider a system with state-action pairs $(\bbX,\bbU)$ along with a permutation $(\tbX,\tbU) = (\bbP^{\Tr}\bbX, \bbP^{\Tr}\bbU)$. The system is said to be permutation invariant if for all $(\bbX,\bbU)$ we have
\begin{equation} \label{eq:sysInvariance}
   c(\bbX, \bbU) = c(\tbX, \tbU).
\end{equation}\end{definition}

For systems that satisfy Definition \ref{def_sys_invariance} we have that a description with state-action pairs $(\tbX,\tbU)$ is, save for a change of labels, equivalent to a system described by state-action pairs $(\bbX,\bbU)$. Ensuring permutation invariance of a system in the sense of Definition \ref{def_sys_invariance} is not difficult but requires care; see Remarks \ref{rmk_sys_invariance_no} and \ref{rmk_sys_invariance_yes}.

Observe now that the product $\tbS= \bbP^T\bbS\bbP$ is a consistent permutation of the rows and columns of the shift operator $\bbS$. As it follows from \eqref{eq:imitationLearning}, a trained GNN is an optimal filter tensor $\ccalH^{*}$. It is then feasible to train the GNN \emph{offline} on a system with state realizations $\bbX$ and shift operator $\bbS$ and execute it \emph{online} in a system with state realizations $\tbX$ and shift operator $\tbS$. This results in actions that are chosen as per
\begin{equation} \label{eqn_transfered_gnn}
    \tbU(t) = \sfPhi\Big(\tbX(t) ; \tbS, \ccalH^{*}\Big)
\end{equation}
We say that the GNN is \emph{transferred.} Alternatively, we can consider the permuted system with state realizations $\tbX$ and shift operator $\tbS$ along with a training process in which we choose the filter tensor $\tilde{\ccalH}^{\ast}$ that is optimal in imitating the centralized policy $\Pi^*_{\text{c}}\big(\tbX(t)\big)$ of the permuted system [cf. \eqref{eq:imitationLearning}],
\begin{equation} \label{eq:imitationLearning_tilde}
   \tilde{\ccalH}^{\ast}
      = \argmin_{\ccalH}\mbE
           \bigg[ \,
              \Big \| \,
                 \sfPhi\big(\tbX(t); \tbS, \ccalH\big)
                    -  \Pi^*_{\text{c}}\big(\tbX(t)\big) \,
                       \Big \| \,
                          \bigg].
\end{equation}
With this tensor available we would choose actions $\tbU(t) = \sfPhi(\tbX(t) ; \tbS, \tilde{\ccalH}^{*})$. The following proposition shows that there is no advantage in doing so because the tensors $\tilde{\ccalH}^{*}$ and $\ccalH^{*}$ are the same.

%
\begin{proposition}[Permutation Invariance] \label{thm:permutationInvariance}
Consider a system described by state-action pair $(\bbX,\bbU)$ and shift operator $\bbS$ along with a permuted system described by $(\tbX,\tbU)=(\bbP^{\Tr}\bbX,\bbP^{\Tr}\bbU)$ and $\tbS=\bbP^{\Tr}\bbS\bbP$. If the system is invariant in the sense of Definition \ref{def_sys_invariance}, the respective optimal imitation tensors [cf. \eqref{eq:imitationLearning} and \eqref{eq:imitationLearning_tilde}] are equivalent,
\begin{equation} \label{eq:permutationInvariance}
   \tilde{\ccalH}^{\ast} \equiv \ccalH^{\ast}.
\end{equation} \end{proposition}
%
%
\begin{proof} See Appendix~\ref{sec:appendix}. \end{proof}

%
Proposition \ref{thm:permutationInvariance} proves that if the filter taps are optimal for a certain network, then they are optimal for all of its permutations. The result is a direct consequence of the permutation equivariance of graph filters \cite[Prop.~1]{Gama20-Stability}, which implies the permutation equivariance of GCNNs \cite[Prop.~2]{Gama20-Stability} and GRNNs \cite[Prop.~1]{Ruiz20-GRNN}.

Proposition \ref{thm:permutationInvariance} is a fundamental enabler of imitation learning because the theorem holds without having access to the permutation matrix $\bbP$. In \eqref{eqn_transfered_gnn} we just execute the GNN that we trained prior to permutation in the permuted shift operator and permuted states. Proposition \ref{thm:permutationInvariance} claims that this is an optimal strategy. Indeed, if we had the clairvoyance of knowing the configuration at execution time we would still find that the same filter tensor is optimal. An important observation to make is that it is unlikely that we will encounter exact permutations at execution and training time. We are, instead, likely to find systems that are close to permutations of systems that we encounter during training. In this scenario, Proposition \ref{thm:permutationInvariance} entails that GNNs are likely to have similar optimal tensor and that there is therefore little loss in using the transferred tensor $\ccalH^*$ instead of the optimal tensor $\tilde{\ccalH}^*$. This statements are formalized in papers that explore stability \cite{Gama20-Stability} and transferability \cite{Ruiz20-Transferability} analyses of GNNs -- see \cite{Gama20-GNNs, Ruiz21-GNNs} for a tutorial presentation.

%
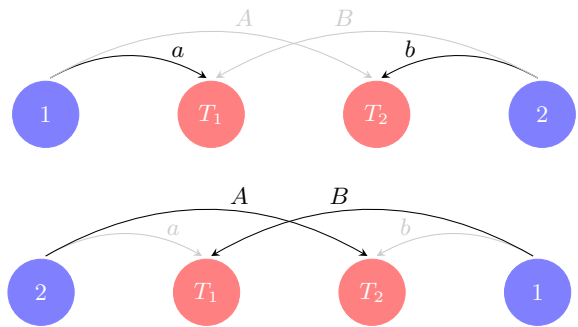
\begin{figure}
\centering
%

\tikzstyle{empty node} = [ circle, 
                     draw = black,
                     text = black, 
                     minimum size = 0.8*\unit]

\tikzstyle{node} = [ empty node, 
                     fill = blue!50,
                     draw = blue!50,
                     text = white]

\tikzstyle{blue node} = [ empty node, 
                         fill = blue!50,
                         draw = blue!50,
                         text = white]

\tikzstyle{red node} = [ empty node, 
                         fill = red!50,
                         draw = red!50,
                         text = white]

\def \myfactor {1.1}
\def \unit {\myfactor cm}

\def \deltanode {(2, 0)}

{\fontsize{9}{9}\selectfont 
\begin{tikzpicture}[scale = \myfactor, -stealth, shorten <= 2pt, shorten >=2pt]
  \node                         [blue node] (N1)  {$1$};
  \path (N1) + \deltanode node  [red node]  (T1)  {$T_1$};
  \path (T1) + \deltanode node  [red node]  (T2)  {$T_2$};
  \path (T2) + \deltanode node  [blue node] (N2)  {$2$};


  \path[black!99] (N1.north) edge [bend left]  node [above, pos = 0.8] {$a$} (T1.north);
  \path[black!99] (N2.north) edge [bend right] node [above, pos = 0.8] {$b$} (T2.north);
  \path[black!20] (N1.north) edge [bend left]  node [above, pos = 0.6] {$A$} (T2.north);
  \path[black!20] (N2.north) edge [bend right] node [above, pos = 0.6] {$B$} (T1.north);

\end{tikzpicture}

\bigskip

\begin{tikzpicture}[scale = \myfactor, -stealth, shorten <= 2pt, shorten >=2pt]
  \node                         [blue node] (N1)  {$2$};
  \path (N1) + \deltanode node  [red node]  (T1)  {$T_1$};
  \path (T1) + \deltanode node  [red node]  (T2)  {$T_2$};
  \path (T2) + \deltanode node  [blue node] (N2)  {$1$};


  \path[black!20] (N1.north) edge [bend left]  node [above, pos = 0.8] {$a$} (T1.north);
  \path[black!20] (N2.north) edge [bend right] node [above, pos = 0.8] {$b$} (T2.north);
  \path[black!99] (N1.north) edge [bend left]  node [above, pos = 0.6] {$A$} (T2.north);
  \path[black!99] (N2.north) edge [bend right] node [above, pos = 0.6] {$B$} (T1.north);

\end{tikzpicture}}
\caption{Permutation Equivariance in System Models. The use of GNNs is warranted in systems that are permutation invariant. Systems with this property are common, but not all decentralized control problems are permutation invariant. See Remarks \ref{rmk_sys_invariance_no} and \ref{rmk_sys_invariance_yes}.}
\label{fig_sys_invariance}
\end{figure}

%
\begin{remark}[Permutation Variant System]\label{rmk_sys_invariance_no}
Constructing costs that satisfy Definition \ref{def_sys_invariance} requires care. Consider the system in Figure \ref{fig_sys_invariance} where we want to drive agent $1$ towards target $T_1$ and agent $2$ towards $T_2$. For a given system configuration the system's reward is the sum of the distances from agent $1$ to target $T_1$ and from agent 2 to target $T_2$. When we consider the configuration shown on the top the reward is $c(\bbX) = a+b$. The configuration shown at the bottom is a relabeling of the agents but the reward that is collected is $c(\tbX) = A+B$ which is in general different. This \emph{labeled} motion planing problem is not amenable to solutions based on GNNs.
\end{remark}

%
\begin{remark}[Permutation Invariant System]\label{rmk_sys_invariance_yes}
An example of a problem that satisfies Definition \ref{def_sys_invariance} is when we want to drive the agents to \emph{either} target. We can do that by choosing as reward the distance to the nearest target. When we do that the reward for the top configuration is $c(\tbX) = b+a$ which is the same. This \emph{unlabeled} motion planing problem is amenable to solutions based on GNNs. See Sec. \ref{sec:pathPlanning} for an example of how proper craft of state and rewards is required to yield a system that is amenable to the use of GNNs.

\end{remark}


\section{Time-Varying Graphs} \label{sec:timeVarying}



{
    \begin{table*}[ht]
        \centering
        \caption{Average (std. deviation) normalized cost for different hyperparameters in the flocking problem. Optimal cost: $52(\pm 1)$.}
        \label{tab:flocking}
        \begin{subtable}{0.28\textwidth}
            \centering
            \begin{tabular}{c|ccc}
                $G$ $\backslash$ $K$ & $2$ & $3$ & $4$  \\ \hline
                $16$ & \cellpurple $10 (\pm 2)$ & \cellred $8 (\pm 2)$ & \cellred $8 (\pm 1)$ \\
                $32$ & \cellpurple $11 (\pm 3)$ & \cellred $8 (\pm 1)$ & \cellred $\mathbf{7 (\pm 1)}$ \\
                $64$ & \cellpurple $10 (\pm 2)$ & \cellred $8 (\pm 1)$ & \cellred $8 (\pm 1)$ \\
            \end{tabular}
            \caption{GF}
            \label{subtab:GF}
        \end{subtable}
        \begin{subtable}{0.34\textwidth}
            \centering
            \begin{tabular}{|ccc}
                $2$ & $3$ & $4$  \\ \hline
                \cellorange $3.4 (\pm 0.3)$ & \cellorange $3.2 (\pm 0.2)$ & \cellorange $3.1 (\pm 0.3)$ \\
                \cellblue $1.88 (\pm 0.07)$ & \cellblue $1.8 (\pm  0.1)$ & \cellblue $1.81 (\pm 0.07)$ \\
                \cellgreen $1.6 (\pm 0.2)$ & \cellgreen $\mathbf{1.60 (\pm  0.08)}$ & \cellgreen $1.63 (\pm 0.08)$ \\
            \end{tabular}
            \caption{GCNN}
            \label{subtab:GCNN}
        \end{subtable}
        \begin{subtable}{0.34\textwidth}
            \centering
            \begin{tabular}{|ccc}
                $2$ & $3$ & $4$  \\ \hline
                \cellorange $2.7 (\pm 0.2)$ & \cellorange $2.6 (\pm 0.1)$ & \cellorange $2.6 (\pm 0.1)$ \\
                \cellblue $1.60 (\pm 0.07)$ & \cellblue $1.58 (\pm  0.07)$ & \cellblue $1.58 (\pm 0.07)$ \\
                \cellgreen $1.48 (\pm 0.05)$ & \cellgreen $\mathbf{1.48 (\pm  0.05)}$ & \cellgreen $1.48 (\pm 0.06)$ \\
            \end{tabular}
            \caption{GRNN}
            \label{subtab:GRNN}
        \end{subtable}
    \end{table*}
}



The parametrizations discussed in Sec.~\ref{sec:GNN}, while suitable for learning decentralized controllers, assume that the trajectory $\{\bbX(t)\}$ is defined always on the same graph described by $\bbS$. However, in many practical instances, the control actions cause a movement of the agents. Such a movement alters the relative location of the agents, and thus, alters the communication network that the agents establish. Changing the communication network means that the graph support changes and thus the graph filter \eqref{eq:graphFilter} may no longer accurately model the communication dynamics. While there are several different graph-time filter alternatives \cite{Isufi17-Random, Grassi18-TimeVertex, Isufi19-Forecasting, Gama19-Control}, in what follows we focus on \emph{unit-delay} graph filters, and we build the corresponding \emph{unit-delay} GCNNs and GRNNs from them.

Let $\{(\bbX(t), \bbS(t))\}$ be a \emph{trajectory} where now each point in the sequence is comprised of a graph signal $\bbX(t)$ and its corresponding support $\bbS(t)$. Consider that each time instant $t$ represents the \emph{exchange clock}. This means that every time a node exchanges information with its neighbors, one time instant passes, creating a \emph{delayed information structure}
\begin{equation} \label{eq:delayedInformation}
\ccalX_{i}(t) = \bigcup_{k=0}^{\infty} \Big\{ \bbx_{j}(t-k) \ , \ j  \in \ccalN_{i}^{k}(t)\Big\}
\end{equation}
where $\ccalN_{i}^{k}(t)$ is the set of nodes $k$ hops away from node $i$, delayed $k$ time instants, and defined recursively as $\ccalN_{i}^{k}(t) = \{j' \in \ccalN_{j}^{k-1}(t-1) \ , \ j \in \ccalN_{i}(t)\}$ with $\ccalN_{i}^{1}(t) = \ccalN_{i}(t)$ and $\ccalN_{i}^{0} = \{i\}$ [cf. \eqref{eq:neighborHistory}]. The collection $\ccalX(t) = \{\ccalX_{i}(t)\}_{i=1,\ldots,N}$ of the delayed information structure at all nodes is the \emph{delayed information history}. The delayed information structure \eqref{eq:delayedInformation} means that each node only has access to past information from its neighbors, and this information gets delayed by the number of hops that had to be traversed to reach such information.

The FIR graph filter can be adapted to respect the delayed information history as follows [cf. \eqref{eq:graphFilter}]
\begin{equation} \label{eq:delayedGraphFilter}
\sfH\big( \ccalX(t) \big) = \sum_{k=0}^{K} \bbS(t) \cdots \bbS(t-(k-1)) \bbX(t-k) \bbH_{k}
\end{equation}
where the set of filter taps $\ccalH = \{\bbH_{k} \ , \ k = 0,\ldots,K\}$ characterizes the operation. Note that the output of \eqref{eq:delayedGraphFilter} is also a graph process defined over the same support sequence $\{\bbS(t)\}$ as the input graph process $\{\bbX(t)\}$. The filter in \eqref{eq:delayedGraphFilter} is usually called a \emph{unit-delay graph filter}.

These delayed graph filters can be used to build a \emph{unit-delay GCNN} as follows [cf. \eqref{eq:GCNN}]
\begin{equation} \label{eq:delayedGCNN}
\sfPhi\big( \ccalX(t) \big) = \sigma \Big(\sfH\big( \ccalX(t) \big) \Big)
\end{equation}
with $\sigma$ a pointwise nonlinearity. The delayed GCNN is characterized by the same set of filter taps $\ccalH$ and its output is also a graph process, but it is a nonlinear map from the input.

Likewise, we can adapt the GRNN \eqref{eq:GRNNhidden}-\eqref{eq:GRNNout} to satisfy the delayed information history. The hidden state becomes
\begin{equation} \label{eq:delayedGRNNhidden}
\bbZ(t) = \sigma\Big( \sfA \big( \ccalX (t) \big) + \sfB \big( \ccalZ(t-1) \big) \Big)
\end{equation}
where both $\sfA$ and $\sfB$ are unit-delay graph filters [cf. \eqref{eq:delayedGraphFilter}], but the second one acting on the delayed information history created by the hidden state sequence $\ccalZ(t) = \{\ccalZ_{i}(t)\}_{i=1,\ldots,N}$ with $\ccalZ_{i}(t) = \cup_{k=0}^{\infty} \{\ccalZ_{j}(t-k) \ , \ j \in \ccalN_{i}^{k}(t) \}$ [cf. \eqref{eq:delayedInformation}]. The output is now
\begin{equation} \label{eq:delayedGRNNout}
\bbU(t) = \rho \Big( \sfC \big( \ccalZ(t) \big) \Big)
\end{equation}
with $\sfC$ a delayed graph convolution [cf. \eqref{eq:delayedGraphFilter}] as well. Note that since $\bbU(t)$ depends on $\ccalZ(t)$ which, in turn, depends on $\ccalX(t)$, then this means that the clock of the output is one time unit delayed with respect to the clock of the input.

In the context of imitation learning, we note that now the expert controller generates trajectories $\{\bbX(t),\bbS(t)\}$ that include both changes in the state as well as changes in the underlying graph (likely due to changes in position). It is evident that the trajectories generated by the expert controller and observed at training time during the imitation learning phase will be different from those observed at deployment. Therefore, it is of paramount importance that the GNN-based controller is capable of transferring to new graph supports. We note that, while it is hard to obtain tight bounds on how much $\bbS(t)$ changes from $\bbS(t-1)$, this is a reasonable assumption since all graphs in the sequence $\{\bbS(t)\}$ come from the same family (in the experiments considered here, these are planar graphs), and as such, their spectral characteristics are similar \cite{Ruiz20-Transferability}.


\section{Flocking} \label{sec:flocking}



In the problem of \emph{flocking} agents are initially flying at random velocities. The goal is to have them all fly at the same velocity while avoiding collisions with each other. Flocking is a canonical problem in decentralized robotics \cite{Tanner03-Stable, Tanner04-Flocking, Tolstaya19-Flocking} that we use to illustrate the potential of GCNNs [cf. \eqref{eq:delayedGCNN}] and GRNNs [cf. \eqref{eq:delayedGRNNhidden}-\eqref{eq:delayedGRNNout}] in learning decentralized scalable controllers. The experiments illustrate three facts:

\begin{list}{}{
                 \setlength{\labelwidth}{0pt}
                 \setlength{\leftmargin}{11pt}
                 \setlength{\labelsep}{0pt}
                 \setlength{\itemsep}{5pt}
                 \setlength{\topsep}{5pt}
                 \setlength{\parskip}{0pt}
              }

    \item[] \textbf{Successful Imitation.} GCNNs and GRNNs are decentralized architectures relying on local information exchanges with neighboring agents. They nevertheless successfully imitate expert, centralized policies that rely on global information.
    \item[] \textbf{Transference.} GCNNs and GRNNs can be transferred from systems drawn at \emph{offline} training into systems observed during \emph{online} execution (cf. Proposition \ref{thm:permutationInvariance}). This is true even if offline and online systems are not exact permutations.
    \item[] \textbf{Transference at Scale.}  GCNNs and GRNNs can be transferred to systems with larger numbers of nodes. This allows leverage of centralized controllers that are computationally feasible for systems with small numbers of nodes only. GCNNs and GRNNs are trained offline on small systems and scaled online to larger systems.
\end{list}

\vspace{3pt}

\noindent We describe the system and training setup followed by experimental results.

\begin{figure*}[ht]
    \hfill
    \begin{subfigure}{0.3\textwidth}
        \centering
        \includegraphics[width = 0.9\textwidth]{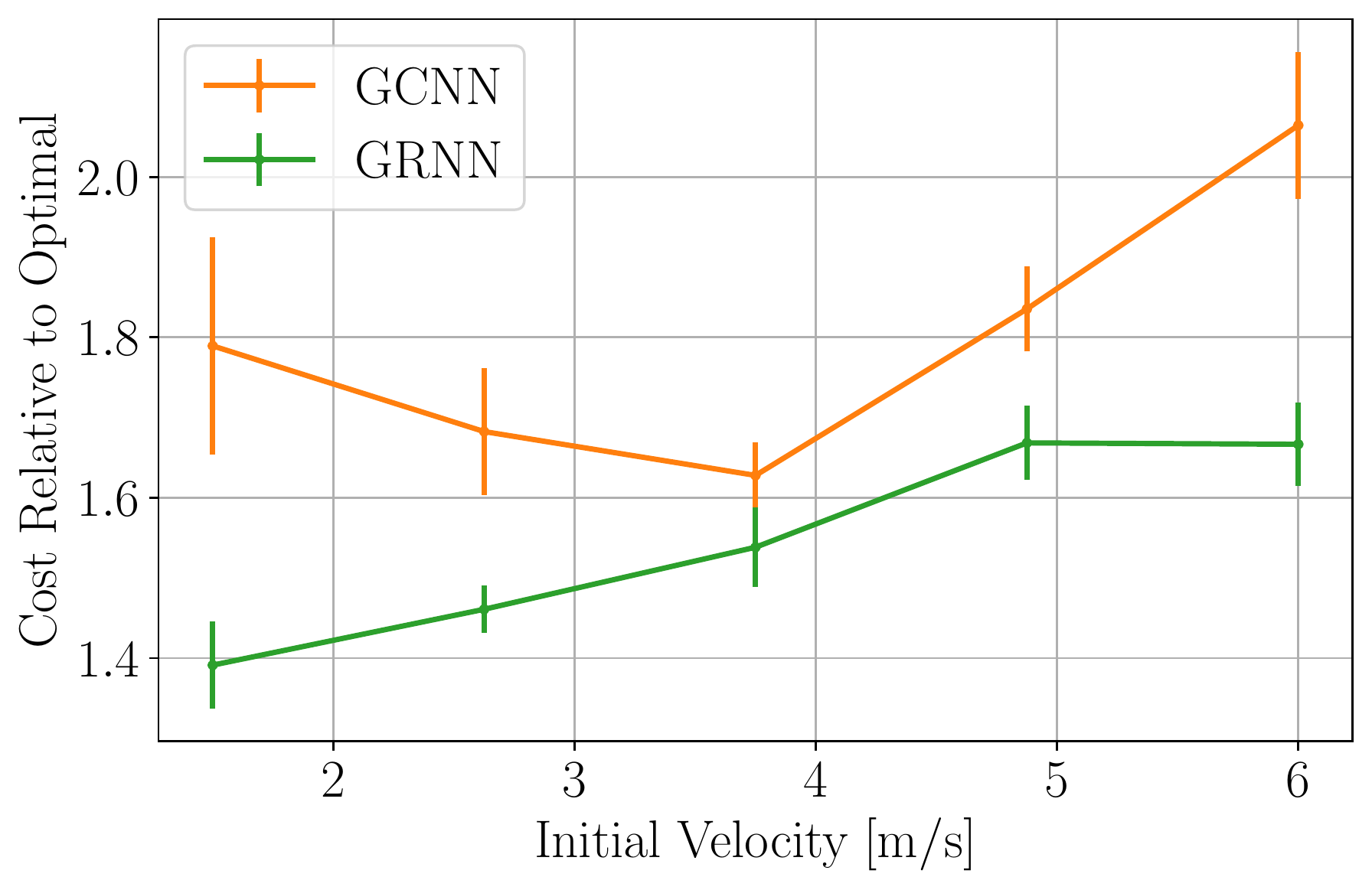}
        \caption{Initial velocity}
        \label{subfig:initVel}
    \end{subfigure}
    \hfill
    \begin{subfigure}{0.3\textwidth}
        \centering
        \includegraphics[width = 0.9\textwidth]{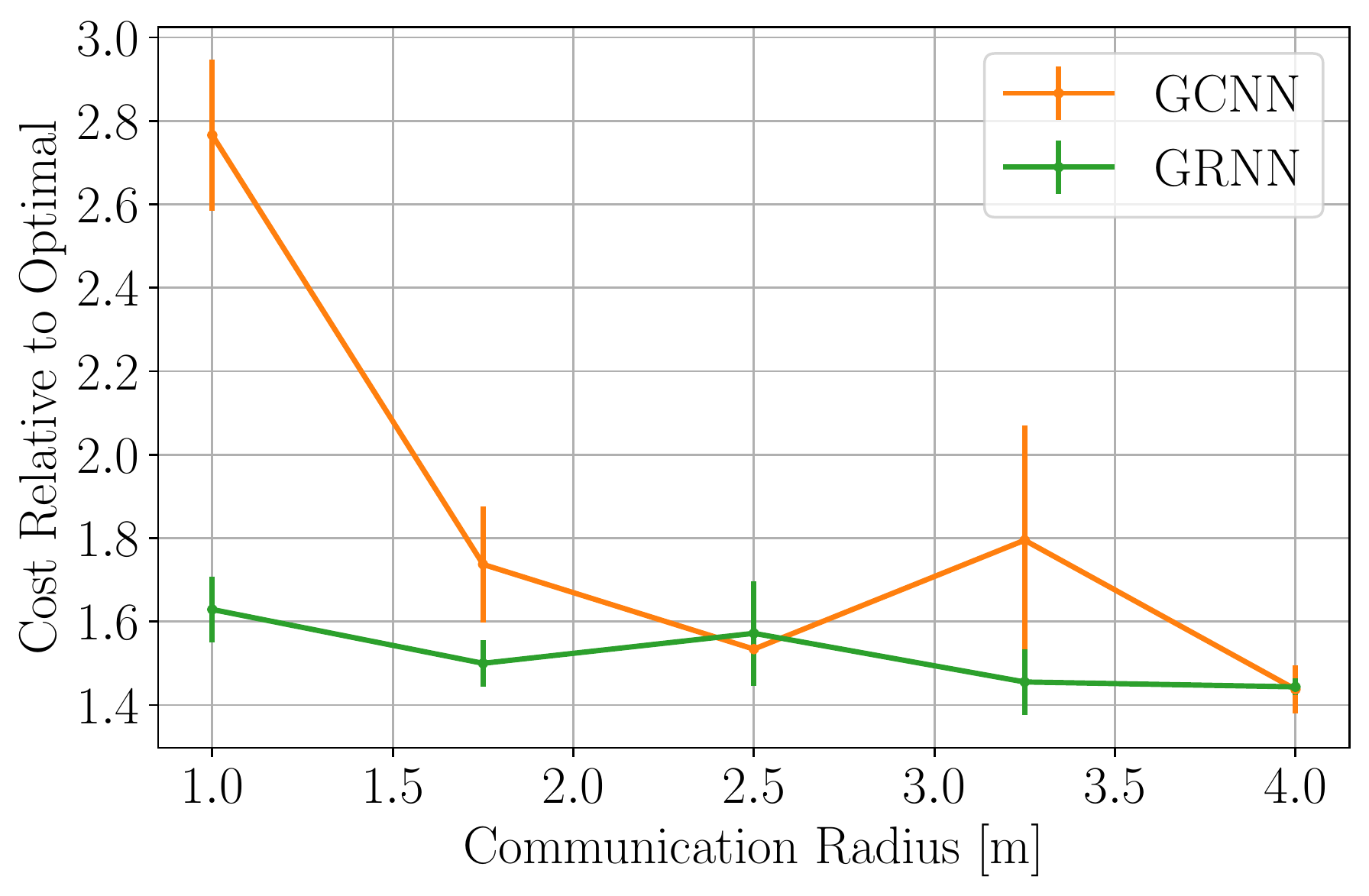}
        \caption{Communication radius}
        \label{subfig:initRadius}
    \end{subfigure}
    \hfill
    \begin{subfigure}{0.3\textwidth}
        \centering
        \includegraphics[width = 0.9\textwidth]{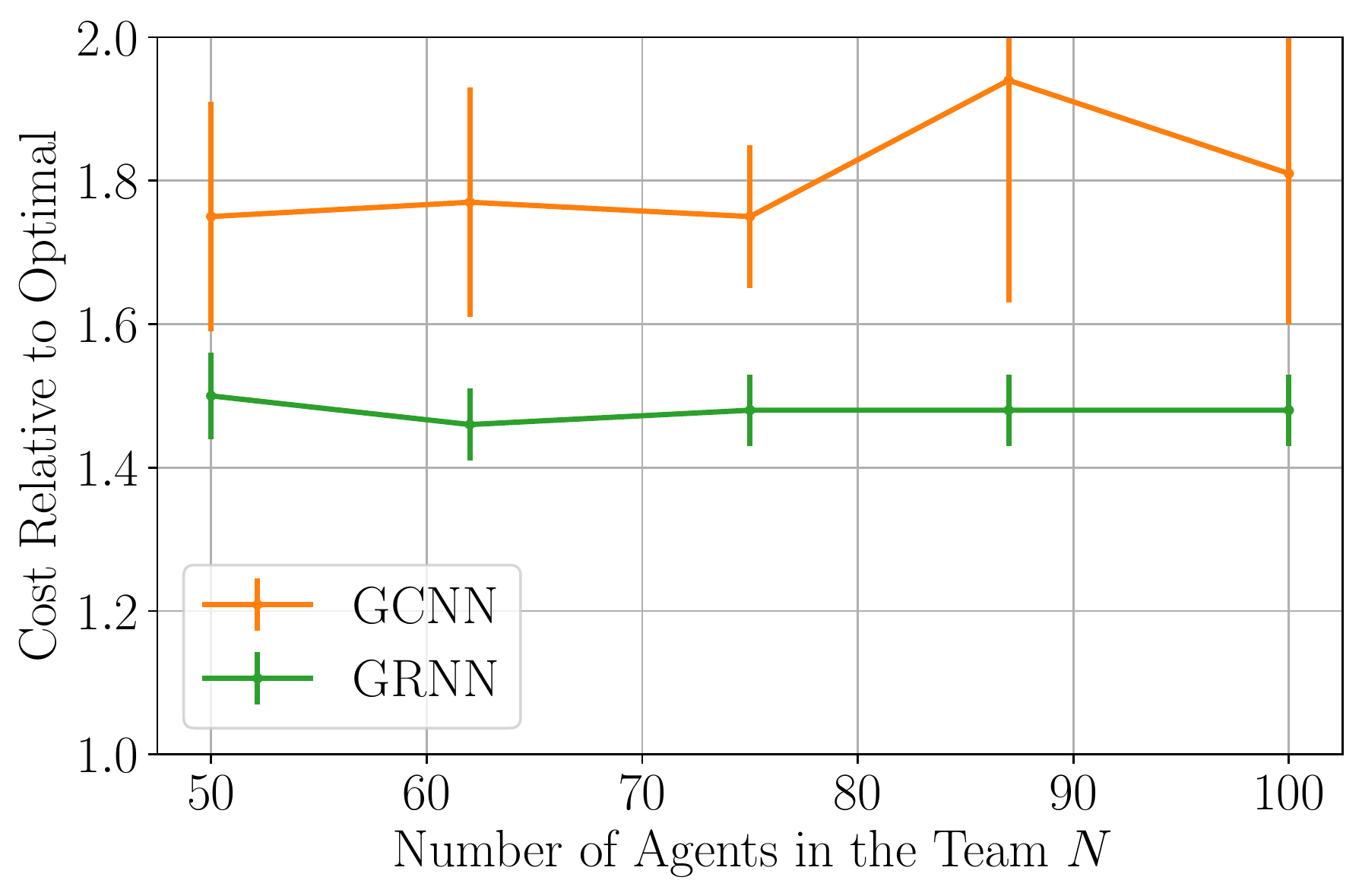}
        \caption{Tranferring at scale}
        \label{subfig:transferScale}
    \end{subfigure}
    \hfill
    \caption{Change in the cost, relative to the optimal cost, for different values of \subref{subfig:initVel} initial velocity, \subref{subfig:initRadius} communication radius, and \subref{subfig:transferScale} size of the team at test time (always trained on teams of size $N=50$). The relative values of the linear graph filter controller exceed $7(\pm 1)$, $8(\pm 1)$, and $8(\pm2)$, respectively, and thus are not shown.}
    \label{fig:flocking}
\end{figure*}

\subsection{System Description and Training}


\myparagraph{Dynamics.} We consider $N$ agents, with $i \in \ccalV$ described by its position $\bbr_{i}(t) \in \reals^{2}$, its velocity $\bbv_{i}(t) \in \reals^{2}$ and its acceleration $\bbu_{i}(t) \in \reals^{2}$. The evolution of the system is
\begin{equation} \label{eq:systemDynamicsFlocking}
\begin{aligned}
    \bbr_{i}(t+1) &= \bbu_{i}(t) T_{\text{s}}^{2}/2 + \bbv_{i}(t) T_{\text{s}} + \bbr_{i}(t) \\
    \bbv_{i}(t+1) &= \bbu_{i}(t) T_{\text{s}} + \bbv_{i}(t)
\end{aligned}
\end{equation}
for $i=1,\ldots,N$. These dynamics imply that the acceleration $\bbu_{i}(t)$ is held constant for the duration of the sampling interval $[tT_{\text{s}},(t+1)T_{\text{s}})$ and that the agents can adjust it instantaneously between sampling intervals. This is a simple model that serves as a proof-of-concept experiment of the use of GNNs to learn decentralized controllers. More involved examples can be found in \cite{Tolstaya19-Flocking, TK20-Visual}. We set $T_{s} = 0.01\text{s}$ and $N=50$ agents. We note that these dynamics are used for simulation purposes, but they are considered unknown under the framework of imitation learning.


\myparagraph{Evaluation.} Formally, the objective of flocking is to determine the accelerations $\{\bbU(t)\}_{t}$ that make the velocities of all agents in the team be the same. We thus evaluate the performance of a given controller in terms of the velocity variation of the team
\begin{equation} \label{eq:flockingObjective}
    c\big( \bbV(t) \big) =  \frac{1}{N} \sum_{i=1}^{N} \Big\| \bbv_{i}(t) - \frac{1}{N} \sum_{j=1}^{N} \bbv_{j}(t) \Big\|_{2}^{2}
\end{equation}
Essentially, the cost \eqref{eq:flockingObjective} is measuring how different the velocity of each individual agent $\bbv_{i}(t)$ is from the average velocity of the team $N^{-1} \sum_{j=1}^{N} \bbv_{j}(t)$. We note that $\sum_{t} c(\bbV(t))$ also measures how long it takes for the system to get controlled, while $c(\bbV(T))$ for some time horizon $T$ gives us an idea of how good the objective was achieved in the end. Note that this cost function satisfies \eqref{eq:sysInvariance}. In all cases, the cost is reported relative to the cost of the expert controller, which is introduced next.


\myparagraph{Expert controller.} A \emph{centralized} controller that avoids collisions is given by \cite{Tanner03-Stable}
\begin{equation} \label{eq:optimalActionFlocking} \nonumber
    \bbu_{i}^{\ast}(t) = - \sum_{j=1}^{N} \Big( \bbv_{i}(t) - \bbv_{j}(t) \Big) - \sum_{j=1}^{N} \nabla_{\bbr_{i}} \sfC\sfA \Big( \bbr_{i}(t), \bbr_{j}(t)\Big)
\end{equation}
where $\nabla_{\bbr_{i}} \sfC\sfA ( \bbr_{i}(t), \bbr_{j}(t))$ is the gradient of the collision avoidance potential $\sfC\sfA:\reals^{2} \times \reals^{2} \to \reals$ with respect to the position of the agent $i$ and evaluated at its position $\bbr_{i}(t)$ and the position of every other agent $\bbr_{j}(t)$ at time $t$. The specific form of the collision avoidance potential $\sfC\sfA$ is given by
\begin{equation} \label{eq:collisionAvoidance} \nonumber
    \sfC\sfA\Big( \bbr_{i}, \bbr_{j}\Big)  =
    \begin{cases}
        1/\|\bbr_{ij}\|_{2}^{2} - \log(\|\bbr_{ij}\|_{2}^{2}) & \text{if} \|\bbr_{ij} \|_{2} \leq R_{\text{CA}} \\
        1/R_{\text{CA}}^{2} - \log(R_{\text{CA}}^{2}) & \text{otherwise}
    \end{cases}
\end{equation}
with $\bbr_{ij} = \bbr_{i} - \bbr_{j}$ and $R_{\text{CA}} > 0$ indicating the minimum acceptable distance between agents. Certainly, $\bbu^{\ast}_{i}(t)$ is a centralized controller since computing it requires agent $i$ to have instantaneous knowledge of $\bbv_{j}(t)$ and $\bbr_{j}(t)$ of every other agent $j$ in the team. We set $R_{\text{CA}} = 1\text{m}$.


\myparagraph{Communication network.} The communication network between agents is determined by their proximity. If agents $i$ and $j$ are within a communication radius $R$ of each other then they are able to establish a link. This builds a communication graph with edge set $\ccalE(t)$ such that $(i,j) \in \ccalE(t)$ if and only if $\|\bbr_{i}(t) - \bbr_{j}(t) \|_{2} \leq R$. We adopt the corresponding binary adjacency matrix as the support $\bbS(t)$. We assume that communication exchanges occur within the interval determined by the sampling time $T_{\text{s}}$, so that the action clock and the communication clock coincide. We note that this is a simplified communication model, but that more involved models accounting for channel losses are possible by simply adjusting the communication graph $\ccalG(t)$ and the corresponding support matrix $\bbS(t)$. We set $R=2\text{m}$.

The state of the system $\bbx_{i}(t) \in \reals^{6}$, which are the messages transmitted among agents, is given by \cite{Tolstaya19-Flocking}
\begin{align} \label{eq:flockingState}
    \bbx_{i}(t) = \bigg[ \quad & \sum_{j : v_{j} \in \ccalN_{i}(t)} \big(\bbv_{i}(t) - \bbv_{j}(t) \big), \\
    & \sum_{j : v_{j} \in \ccalN_{i}(t)} \frac{\bbr_{ij}(t)}{\| \bbr_{ij}(t)\|_{2}^{4}} ,
    \sum_{j:v_{j} \in \ccalN_{i}(t)} \frac{\bbr_{ij}(t)}{\| \bbr_{ij}(t)\|_{2}^{2}} \quad \bigg].\nonumber
\end{align}
Note that the value of the state \eqref{eq:flockingState} for each agent can be computed locally, depending on the relative distance and velocity between each agent and its immediate neighbors. The state \eqref{eq:flockingState} is used as the input to the GNN-based controllers.


\myparagraph{Decentralized learning architectures.} The communication network imposes a delayed information structure $\ccalX(t)$ [cf. \eqref{eq:delayedInformation}], and thus we adopt models based on unit-delay filters [cf. Sec.~\ref{sec:timeVarying}]. First, we consider linear graph filters (GF) $\sfH(\ccalX(t))$ given by \eqref{eq:delayedGraphFilter}. More specifically, we consider a cascade of two filters, the first one producing $F_{1}=G$ features after $K_{1} = K$ communications, and the second one combining those $G$ features into the $F_{2}= 2$ output features (with $K_{2}=0$), corresponding to the control action $\bbu_{i}(t)$ taken by each agent. Second, a two-layer GCNN $\sfPhi_{\text{GCNN}}(\ccalX(t))$ given by \eqref{eq:delayedGCNN}, with $F_{1}=G$, $K_{1} = K$, $F_{2} = 2$, $K_{2}=0$ and a hyperbolic tangent nonlinearity $\sigma(x) = \tanh(x)$. Third, GRNNs $\sfPhi_{\text{GRNN}}(\ccalX(t), \ccalZ(t))$, with $H=G$ and $K$ for the hidden state \eqref{eq:delayedGRNNhidden}, and $2$ output features and $K_{o}=0$ for \eqref{eq:delayedGRNNout}. The values of $G$ and $K$ are set separately for each controller as discussed in Experiment~1.


{
    \begin{table*}[ht]
        \centering
        \caption{Average (std. deviation) evaluation measures for different hyperparameters of the learned architectures. The first column of tables shows error rate, while the second column shows the increase in flowtime, relative to the expert controller.}
        \label{tab:pathPlanning}
        \begin{subtable}{0.49\textwidth}
            \centering
            \begin{tabular}{c|ccc}
                $G$ $\backslash$ $K$ & $2$ & $3$ & $4$  \\ \hline
                $16$ & \cellorange $0.15 (\pm 0.09)$ & \cellorange $0.15 (\pm 0.09)$ & \cellorange $0.15 (\pm 0.09)$ \\
                $32$ & \cellorange $0.16 (\pm 0.08)$ & \cellorange $0.16 (\pm 0.08)$ & \cellorange $0.16 (\pm 0.08)$ \\
                $64$ & \cellblue $0.14 (\pm 0.09)$ & \cellorange $0.16 (\pm 0.08)$ & \cellblue $\mathbf{0.13 (\pm 0.09)}$ \\
            \end{tabular}
            \caption{GF -- Error rate}
            \label{subtab:SRGF}
        \end{subtable}
        \begin{subtable}{0.49\textwidth}
            \centering
            \begin{tabular}{c|ccc}
                $G$ $\backslash$ $K$ & $2$ & $3$ & $4$  \\ \hline
                $16$ & \cellpurple $0.093 (\pm 0.009)$ & \cellred $0.092 (\pm 0.009)$ & \cellred $0.090 (\pm 0.009)$ \\
                $32$ & \cellred $0.091 (\pm 0.009)$ & \cellpurple $0.093 (\pm  0.009)$ & \cellred $0.091 (\pm 0.009)$ \\
                $64$ & \cellred $0.092 (\pm 0.009)$ & \cellred $0.091 (\pm  0.009)$ & \cellred $\mathbf{0.090 (\pm 0.009)}$ \\
            \end{tabular}
            \caption{GF -- Flowtime increase}
            \label{subtab:FTGF}
        \end{subtable}

        \begin{subtable}{0.49\textwidth}
            \centering
            \begin{tabular}{c|ccc}
                $G$ $\backslash$ $K$ & $2$ & $3$ & $4$  \\ \hline
                $16$ & \cellblue $0.13 (\pm 0.09)$ & \cellblue $0.13 (\pm 0.09)$ & \cellblue $0.14 (\pm 0.09)$ \\
                $32$ & \cellgreen $0.09 (\pm 0.09)$ & \cellblue $0.11 (\pm 0.09)$ & \cellgreen $0.10 (\pm 0.09)$ \\
                $64$ & \cellgreen $0.10 (\pm 0.09)$ & \cellgreen $\mathbf{0.09 (\pm 0.09)}$ & \cellblue $0.11 (\pm 0.09)$ \\
            \end{tabular}
            \caption{GCNN -- Error rate}
            \label{subtab:SRGCNN}
        \end{subtable}
        \begin{subtable}{0.49\textwidth}
            \centering
            \begin{tabular}{c|ccc}
                $G$ $\backslash$ $K$ & $2$ & $3$ & $4$  \\ \hline
                $16$ & \cellorange $0.087 (\pm 0.009)$ & \cellblue $0.085 (\pm 0.009)$ & \cellblue $0.081 (\pm 0.008)$ \\
                $32$ & \cellgreen $0.078 (\pm 0.008)$ & \cellgreen $0.073 (\pm  0.007)$ & \cellgreen $0.075 (\pm 0.008)$ \\
                $64$ & \cellgreen $0.073 (\pm 0.009)$ & \cellgreen $\mathbf{0.071 (\pm  0.007)}$ & \cellgreen $0.071 (\pm 0.008)$ \\
            \end{tabular}
            \caption{GCNN -- Flowtime increase}
            \label{subtab:FTGCNN}
        \end{subtable}

        \begin{subtable}{0.49\textwidth}
            \centering
            \begin{tabular}{c|ccc}
                $G$ $\backslash$ $K$ & $2$ & $3$ & $4$  \\ \hline
                $16$ & \cellpurple $0.34 (\pm 0.07)$ & \cellpurple $0.26 (\pm 0.07)$ & \cellpurple $0.22 (\pm 0.09)$ \\
                $32$ & \cellred $0.17 (\pm 0.08)$ & \cellred $0.19 (\pm 0.09)$ & \cellpurple $0.21 (\pm 0.08)$ \\
                $64$ & \cellgreen $\mathbf{0.10 (\pm 0.09)}$ & \cellblue $0.13 (\pm 0.09)$ & \cellblue $0.11 (\pm 0.09)$ \\
            \end{tabular}
            \caption{GRNN -- Error rate}
            \label{subtab:SRGRNN}
        \end{subtable}
        \begin{subtable}{0.49\textwidth}
            \centering
            \begin{tabular}{c|ccc}
                $G$ $\backslash$ $K$ & $2$ & $3$ & $4$  \\ \hline
                $16$ & \cellpurple $0.10 (\pm 0.01)$ & \cellpurple $0.09 (\pm 0.01)$ & \cellpurple $0.09 (\pm 0.01)$ \\
                $32$ & \cellblue $0.085 (\pm 0.009)$ & \cellblue $0.082 (\pm  0.008)$ & \cellorange $0.088 (\pm 0.009)$ \\
                $64$ & \cellgreen $\mathbf{0.078 (\pm 0.008)}$ & \cellblue $0.084 (\pm  0.008)$ & \cellorange $0.088 (\pm 0.009)$ \\
            \end{tabular}
            \caption{GRNN -- Flowtime increase}
            \label{subtab:FTGRNN}
        \end{subtable}
    \end{table*}
}



\myparagraph{Training.} The dataset is comprised of $400$ trajectories for training, $20$ for validation and $20$ for testing. Each trajectory is generated by positioning the agents at random in a circle such that their minimum initial distance is $0.1\text{m}$ and their initial velocities are picked also at random from the interval $[-3,3]\text{m}/\text{s}$ in each direction. We note that a bias velocity, also picked at random from $[-3,3]\text{m}/\text{s}$ is included so as to avoid the flocking velocity to be zero. The trajectories are of duration $2\text{s}$ and the maximum acceleration is $10\text{m}/\text{s}^{2}$. The models are trained for $30$ epochs with a batch size of at most $20$ trajectories, totaling $600$ training steps. Every $5$ training steps, the cost \eqref{eq:flockingObjective} is evaluated on the validation set, avoiding the need to compute the cost at every single training step, thereby saving on computational cost. After training has finished, we retain the model characterized by the set of parameters that has achieved the lowest cost \eqref{eq:flockingObjective} on the validation set to avoid overfitting \cite[Ch. 28]{Murphy12-ProbabilisticML}. We solve the imitation learning problem \eqref{eq:imitationLearning} using the ADAM algorithm \cite{Kingma15-ADAM} with learning rate $5 \cdot 10^{-4}$ and forgetting factors $0.9$ and $0.999$. The loss function used for the imitation learning is the mean squared error between the output of the model and the optimal control action. The evaluation measure is the cost \eqref{eq:flockingObjective}. We repeat the simulations for $10$ realizations of the dataset and report the average cost as well as the standard deviation.

\subsection{Experimental Results}


\myparagraph{Experiment 1: Successful Imitation.} First, we test different values of features $G \in \{16, 32, 64\}$ and filter taps $K \in \{2,3,4\}$. Results are shown in Table \ref{tab:flocking}. Note that the GF is characterized by $6G(K+1)+2G$ learnable parameters, the same as the GCNN, while the GRNN learns $(6G+G^{2})(K+1)+2G$ parameters. We see that the linear graph filter has a performance that is five times worse than the nonlinear architectures. This is because we know that even for simple linear problems, the optimal decentralized solution is nonlinear \cite{Witsenhausen68-Counterexample}, and thus cannot be captured by a linear model. Then we see that the GRNN exhibits the best performance, and that the GCNN comes in second with reasonably good performance as well. We also observe that more features $G$ improves performance in this range, but not necessarily larger $K$. From this simulation we select the best pair $(G,K)$ for each of the three architectures and keep them for the following experiments. For these selected values, we note that, at the end time $T$, the velocity variation of the team [cf. \eqref{eq:flockingObjective}] is $0.0132(\pm0.043)$ for the GCNN and $0.0116(\pm0.0029)$ for the GRNN, showing that the team is successfully flocking together. More specifically, this is evidenced by the fact that the velocity is $0.11\text{m}/\text{s}$ within the average velocity of the team, which is in the order of $3\text{m}/\text{s}$, thus showing less than $4\%$ flocking error in the end.


\myparagraph{Experiment 2: Transference.} Second, we run tests for different initial conditions, namely different initial velocities (Fig.~ \ref{subfig:initVel}) and different communication radius (Fig.~\ref{subfig:initRadius}). These experiments test the robustness of the architectures to different initial conditions. We observe in Fig.~\ref{subfig:initVel} that larger initial velocities implies harder to control flocks, and thus the performance decreases as the initial velocities grow. Nevertheless, the GRNN seems to be more robust than the GCNN. With respect to the communication radius, we observe in Fig.~\ref{subfig:initRadius} that the larger the communication radius, the easier the flock is to control. This is expected since more agents can be reached and thus information travels faster with less delay. Again, the more robust architecture is the GRNN.


\myparagraph{Experiment 3: Transference at scale.} As a third and final experiment, we run a test on transferring at scale. We train the architectures for $50$ agents, but then we test them on $N \in \{50, 62, 75, 87, 100\}$ agents. Results are shown in Fig.~\ref{subfig:transferScale}. We observe that both nonlinear architectures (GCNN and GRNN) have good scalability, with the GRNN being virtually perfect, i.e. keeping the same performance as the number of agents increases. This is due to their equivariance and stability properties [cf. Sec.~\ref{sec:permutation}]. In essence, this last experiment shows that it is possible to learn a decentralized controller in a small network setting and then, once trained, transfer this solution to larger networks, without any re-training, successfully scaling.


\begin{remark}\normalfont
The use of GNNs to learn controllers for the flocking problem we consider in this section was first introduced in \cite{Tolstaya19-Flocking}. In this paper, we further consider a GRNN architecture that takes into account the time-variability of the problem and, as such, is shown to work better than the GCNN for a wide range of different initial velocities (Fig.~\ref{subfig:initVel}) and communication radii (Fig.~\ref{subfig:initRadius}). Furthermore, we carry out a comprehensive hyperparameter study (Table~\ref{tab:flocking}) to better understand the dependence of the architectures with each design choice. More importantly, we also show, in a detailed experiment, the transferability at scale (Fig.~\ref{subfig:transferScale}) of both GCNN and GRNN-based controllers, showing their capability of being trained in small teams (where an expert controller is available) and tested in larger teams without loss of performance (where an expert controller is unavailable).
\end{remark}


\section{Path Planning} \label{sec:pathPlanning}


Efficient and collision-free path planning in multi-agent systems is fundamental to advancing mobility. The overarching aim is to generate jointly collision-free paths leading agents from their start positions to designated goal positions. In the discrete domain, this problem is generally referred to as Multi-Agent Path Finding (MAPF). Coupled centralized approaches, which consider the joint configuration space of all involved agents, have the advantage of producing optimal and complete plans, yet tend to be computationally expensive. Indeed, solving for optimality is NP-hard~\cite{Yu13-IntractablePathPlanning}, and although significant progress has been made towards alleviating the computational load~\cite{Sharon15-CBS, Ferner13-ODrM}, it still scales poorly in environments with a high number of potential path conflicts.

\begin{figure*}[ht]
    \hfill
    \begin{subfigure}{0.3\textwidth}
        \centering
        \includegraphics[height=0.559\textwidth, width = 0.9\textwidth]
        {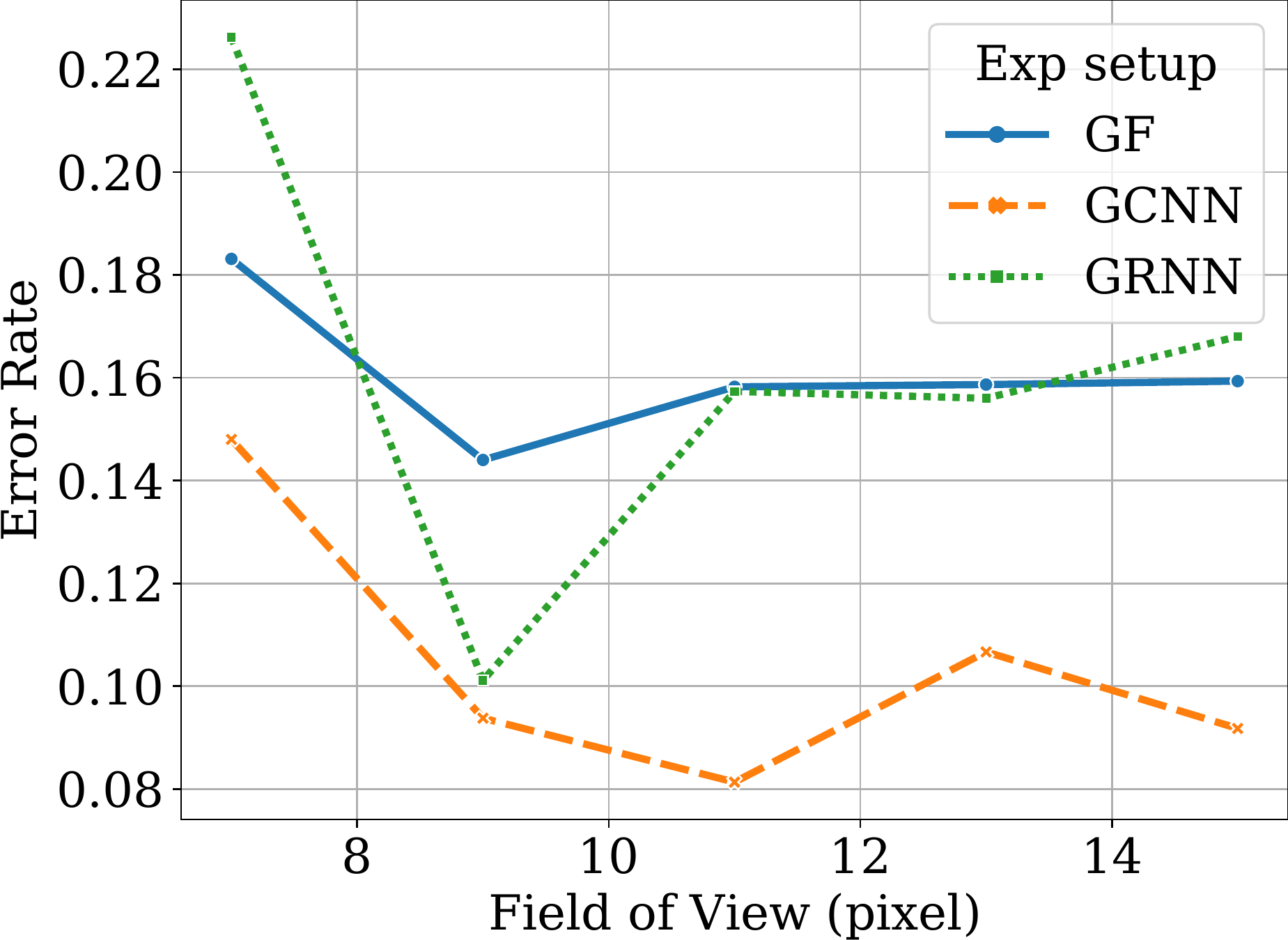}
        \caption{Field of view -- Success rate}
        \label{subfig:SRFOV}
    \end{subfigure}
    \hfill
    \begin{subfigure}{0.3\textwidth}
        \centering
        \includegraphics[height=0.559\textwidth, width = 0.9\textwidth]
        {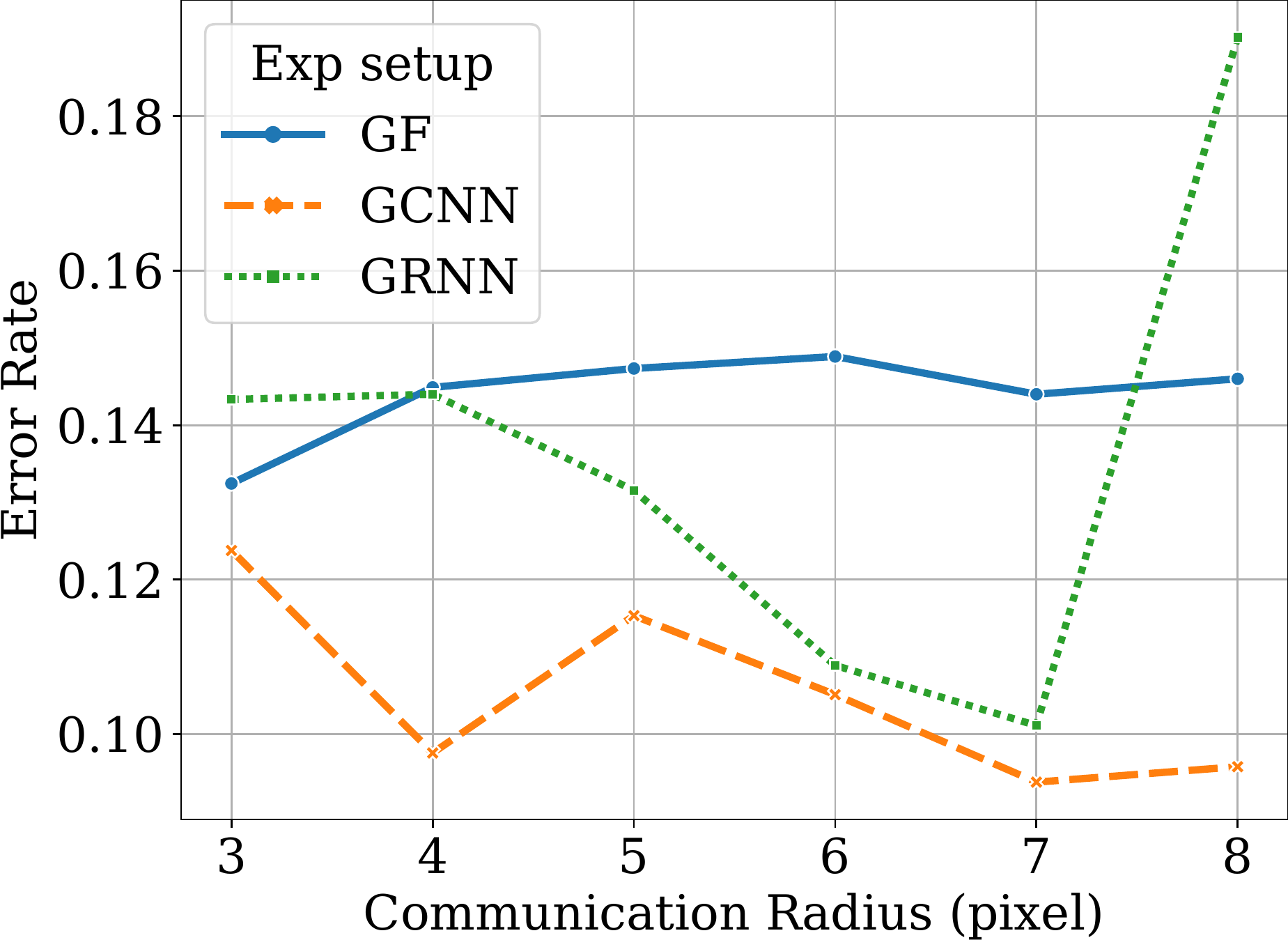}
        \caption{Communication radius -- Success rate}
        \label{subfig:SRradius}
    \end{subfigure}
    \hfill
    \begin{subfigure}{0.3\textwidth}
        \centering
        \includegraphics[height=0.559\textwidth, width = 0.9\textwidth]
        {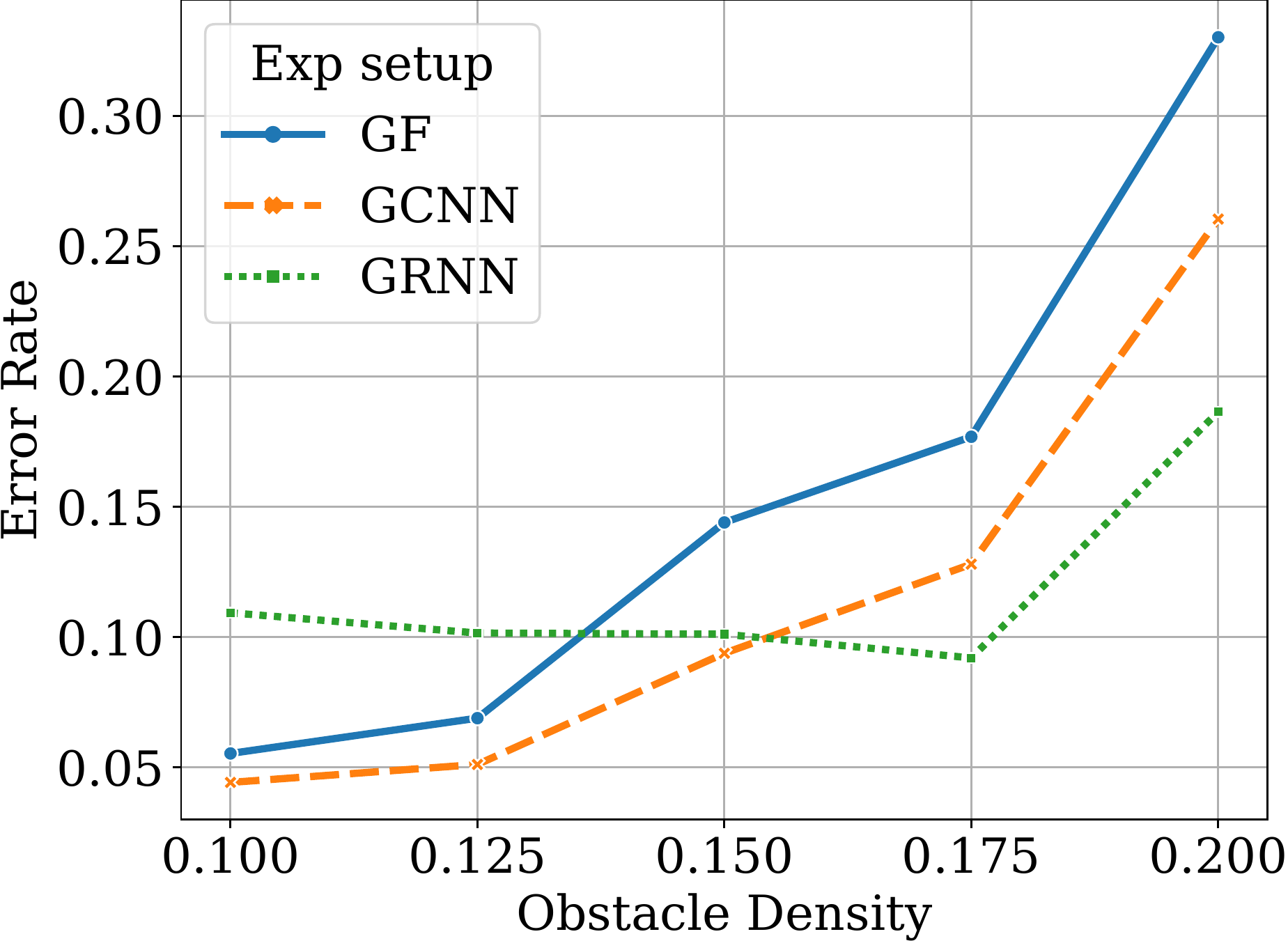}
        \caption{Obstacle density -- Success rate}
        \label{subfig:SRdensity}
    \end{subfigure}
    \hfill

    \hfill
    \begin{subfigure}{0.3\textwidth}
        \centering
        \includegraphics[height=0.559\textwidth, width = 0.9\textwidth]
        {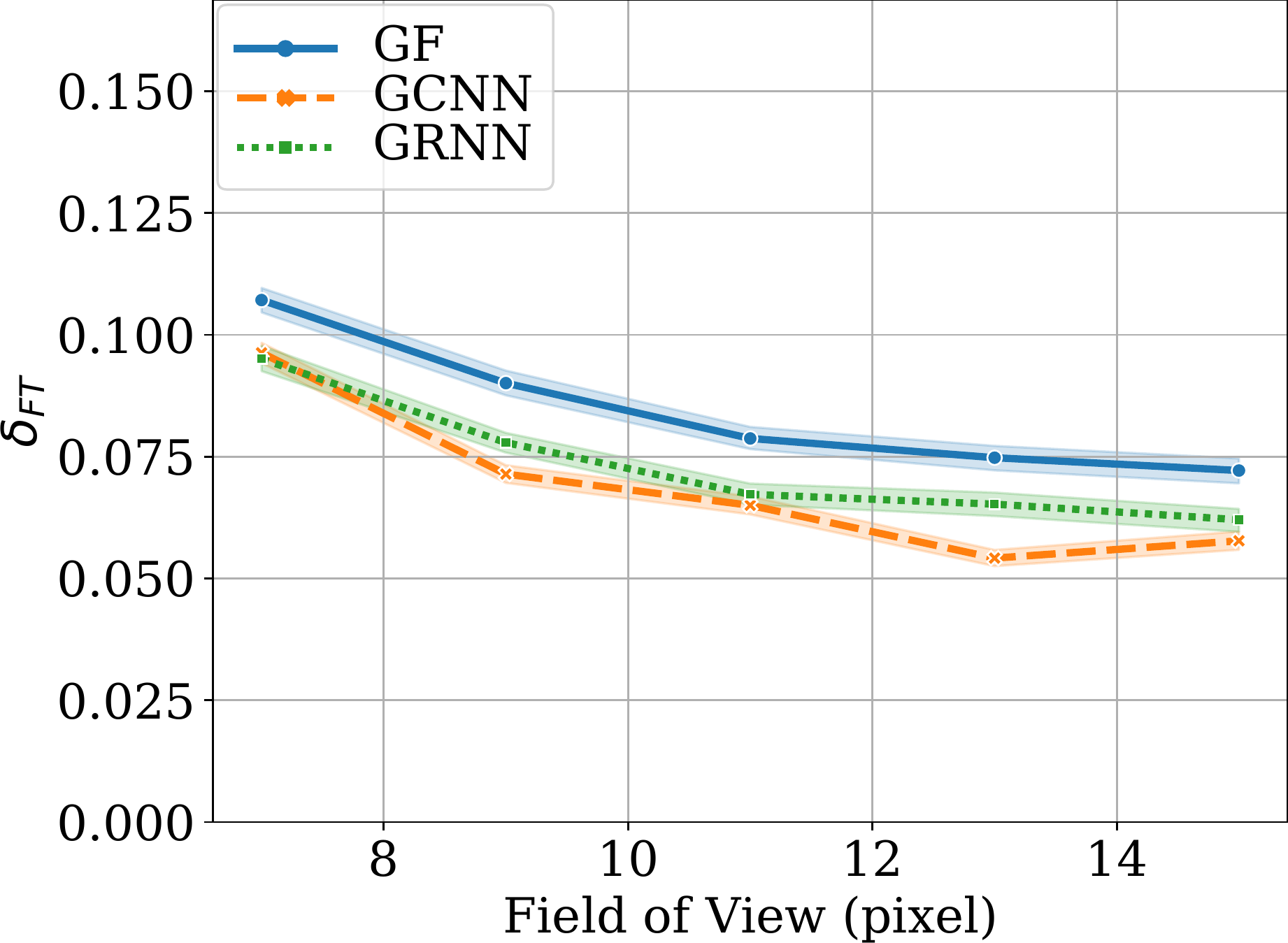}
        \caption{Field of view -- Flowtime increase}
        \label{subfig:FTFOV}
    \end{subfigure}
    \hfill
    \begin{subfigure}{0.3\textwidth}
        \centering
        \includegraphics[height=0.559\textwidth, width = 0.9\textwidth]
        {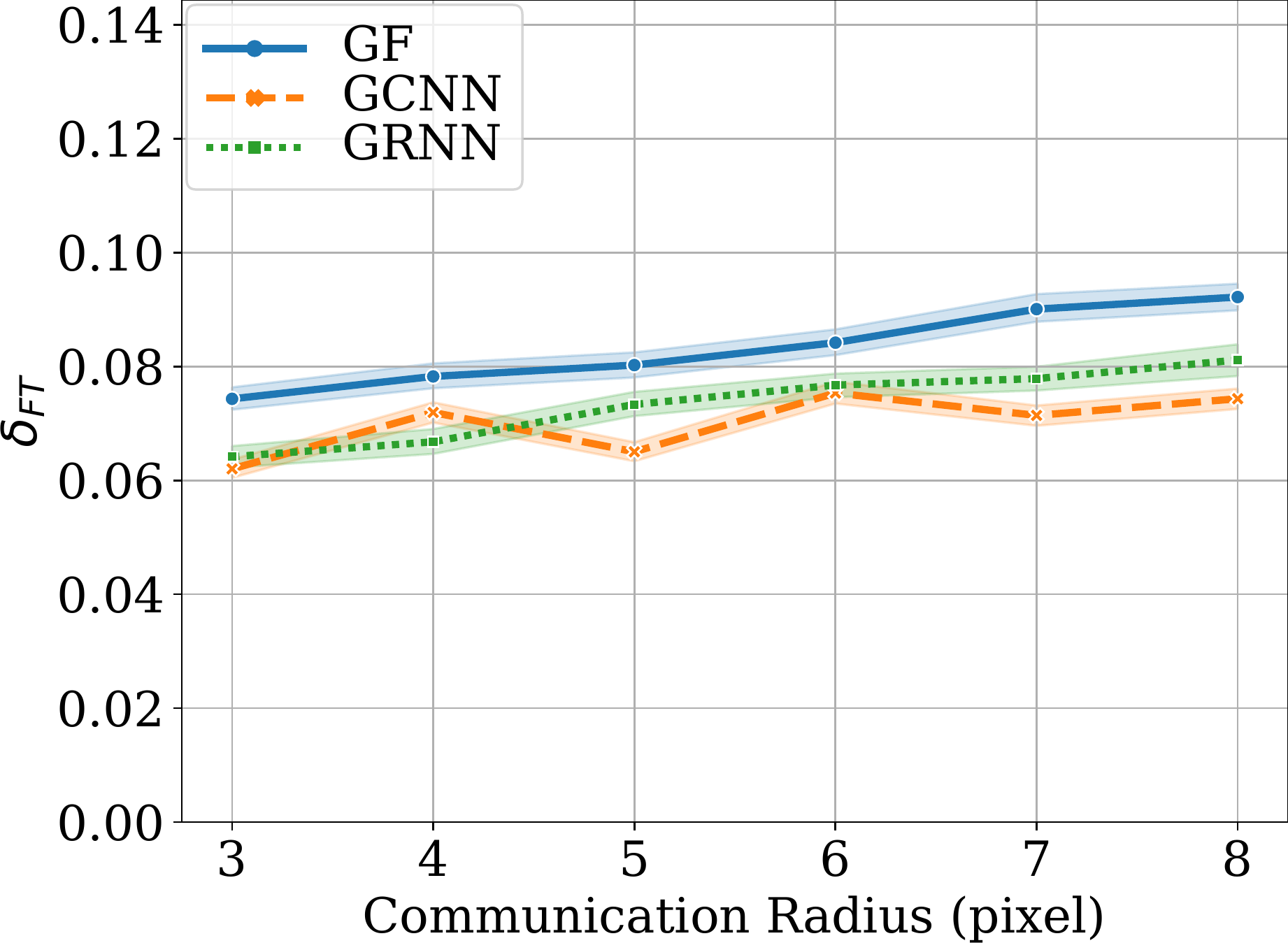}
        \caption{Communication radius -- Flowtime increase}
        \label{subfig:FTradius}
    \end{subfigure}
    \hfill
    \begin{subfigure}{0.3\textwidth}
        \centering
        \includegraphics[height=0.559\textwidth, width = 0.9\textwidth]
        {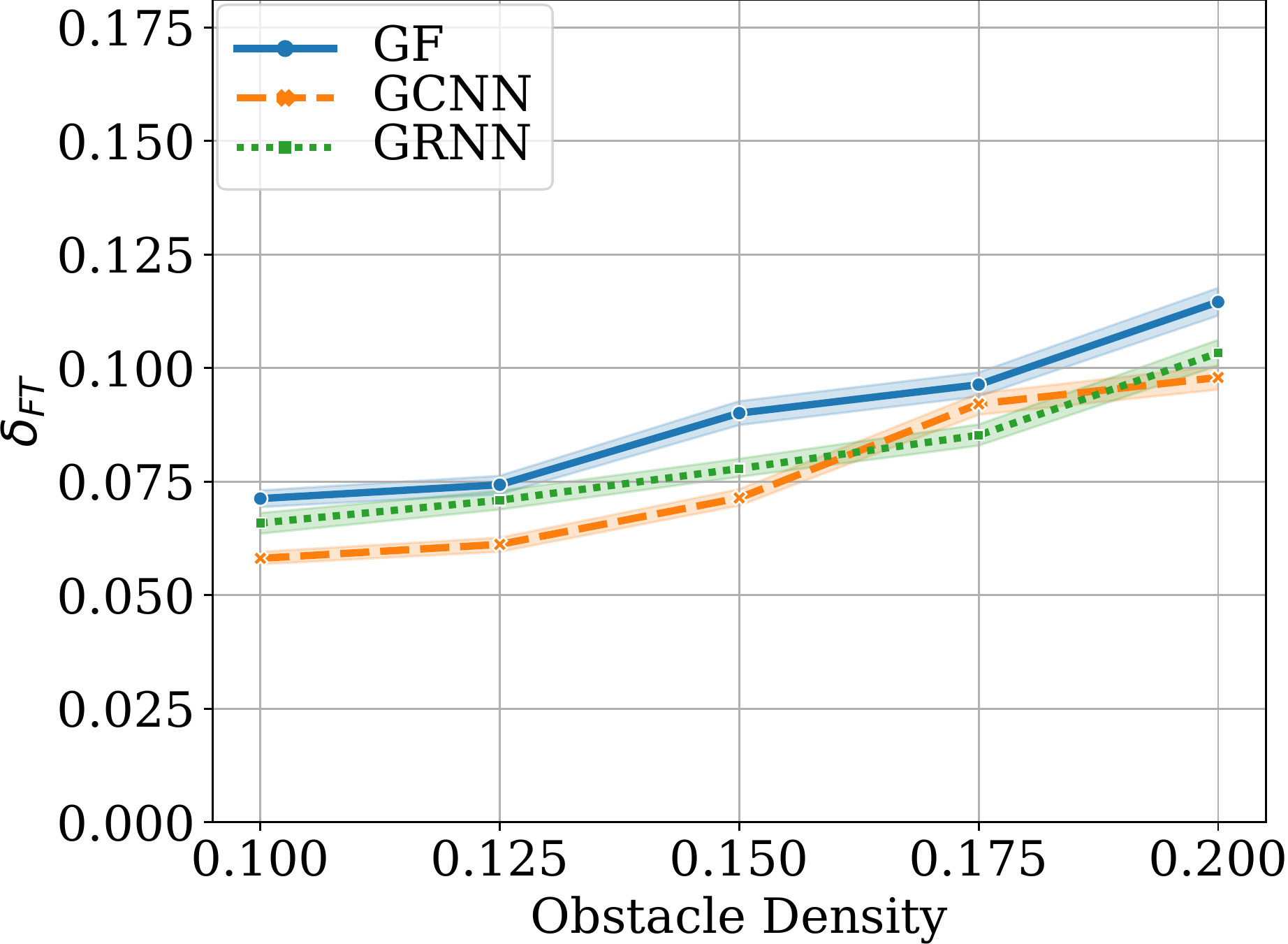}
        \caption{Obstacle density -- Flowtime increase}
        \label{subfig:FTdensity}
    \end{subfigure}
    \hfill
    \caption{Top: Change in the success rate $\alpha$ for different values of \subref{subfig:SRFOV} field of vision, \subref{subfig:SRradius} communication radius, and \subref{subfig:SRdensity} obstacle density. Bottom: Change in flow time increase $\delta_{\text{FT}}$, for different values of \subref{subfig:FTFOV} field of view, \subref{subfig:FTradius} communication radius, and \subref{subfig:FTdensity} obstacle density.}
    \label{fig:pathPlanningInit}
\end{figure*}

\subsection{System Description and Training}



\myparagraph{Dynamics.} Consider a $2D$ finite grid, where agents can move to adjacent positions in the grid at every time instant, as given by a \emph{decision clock} $t = 0,1,\ldots$. This $2D$ grid contains some positions that are blocked and to which the agents cannot access. These positions are called \emph{obstacles} and an agent attempting to move into such a position would incur in a \emph{collision}. Each agent can only partially observe the immediate grid, as given by a rectangle of size $W_{\text{FOV}} \times H_{\text{FOV}}$, centered at the agent. They further know the relative direction of the target. Adopting a local map centered at the agent and providing the relative distance to the target leads to a cost function that satisfies Def.~\ref{def_sys_invariance}. Note that this would not be possible if a centralized map is given to each agent, since the distance to the target would depend on the labeling of the agent. At every instant of the decision clock, an agent takes an action $\bbu_{i}(t)$, moving into one of the four adjacent positions in the grid (or deciding to stay put), with the objective of moving towards its goal on a collision-free path. Note that there are only five possible actions. We thus formulate the multi-agent path planning problem as a sequential decision-making problem~\cite{Li20-Planning, Li20-Message}. We consider $N=10$ agents, navigating a $20 \times 20$  grid with a $15\%$ obstacle density.


\myparagraph{Evaluation.} We measure the performance following two metrics. First, the \emph{error rate} $\alpha$ given by the ratio of unsuccessful cases over the total number of tested cases. We consider a case to be unsuccessful if \emph{any} of its agents fail to arrive at their destinations before the set timeout. Note that the expert controller (to be described next) has an error rate of $0$ as it is always successful. Second, we use the flowtime increase $\delta_{\text{FT}}$, which measures the excess time that takes to complete the objective $FT$, relative to the time taken by the expert controller $FT^{\ast}$, that is $\delta_{\text{FT}} = (FT - FT^{\ast})/FT^{\ast}$.


\myparagraph{Expert controller.} To learn the decentralized controllers by means of imitation learning [cf. \eqref{eq:imitationLearning}], we utilize a controller $\bbU^{\ast}(t)$ called Conflict-Based Search (CBS)~\cite{Sharon15-CBS}. CBS finds a solution by doing tree-based search on the space of possible paths, resulting in a computationally costly controller that is only applicable to small teams. We use it only at training time to generate the set of optimal, collision-free paths.



\myparagraph{Communication network.} As is the case in Sec.~\ref{sec:flocking}, the communication network is given by a graph $\ccalG(t) = (\ccalV, \ccalE(t))$, where two agents are able to communicate between themselves if they are close, i.e. $(i,j) \in \ccalE(t)$ if and only if $\|\bbr_{i}(t) - \bbr_{j}(t)\|_{2} \leq R$, with $\bbr_{i}(t)$ the position of agent $i$ at time $t$, and for some communication radius $R$. We adopt the binary adjacency matrix $\bbS(t)$ as the shift operator. We note that, unlike Sec.~\ref{sec:flocking}, for each time $t$ that represents the decision clock, there can be arbitrarily many communications exchanges. That is, the communication clock can be arbitrarily faster than the decision clock. This makes sense in the current problem, since for each decision, the agents move one position in the grid, and can wait statically until they make the following decision.


\myparagraph{Decentralized learning architectures.} We consider a joint learning architecture whereby each agent processes the local map by means of a regular CNN, and then uses the resulting features as input to a decentralized architecture, i.e. a graph filter, a GCNN or a GRNN. In short, by jointly training both architectures, the system learns to extract the best features as it pertains to the communication network between the agents. Note that, since the decision clock $t$ is slower than the communication clock, then the implementation of a graph filter can be done as in \eqref{eq:graphFilter}, for each time instant $t$. The CNN consists of five layers, all having filters with a kernel of size $3$ with unit stride and zero-padding, followed by a ReLU nonlinearity $\sigma(x) = \max\{0,x\}$; a max pooling of size $2$ is used in the odd-numbered layers. Then, for the decentralized controller, we consider first, a linear graph filter with two filters, one consisting of $F_{1}=G$ features and $K_{1}=K$ hops, and the second layer is the readout layer with $F_{2}=5$ and $K_{2}=0$. The second decentralized learning architecture is the two-layer GCNN with $F_{1}=G$, $K_{1}=K$, $F_{2}=5$ and $K_{2}=0$; and with $\sigma(x)=\tanh(x)$. Finally, the third model is a GRNN with $H=G$ and $K$ for the hidden state \eqref{eq:delayedGRNNhidden}, and $2$ output features and $K_{o}=0$ for \eqref{eq:delayedGRNNout}. The values of $G$ and $K$ are set separately for each controller as discussed in Experiment~1.

\begin{figure*}[ht]
    \hfill
    \begin{subfigure}{0.3\textwidth}
        \centering
        \includegraphics[width = 0.9\textwidth]
        {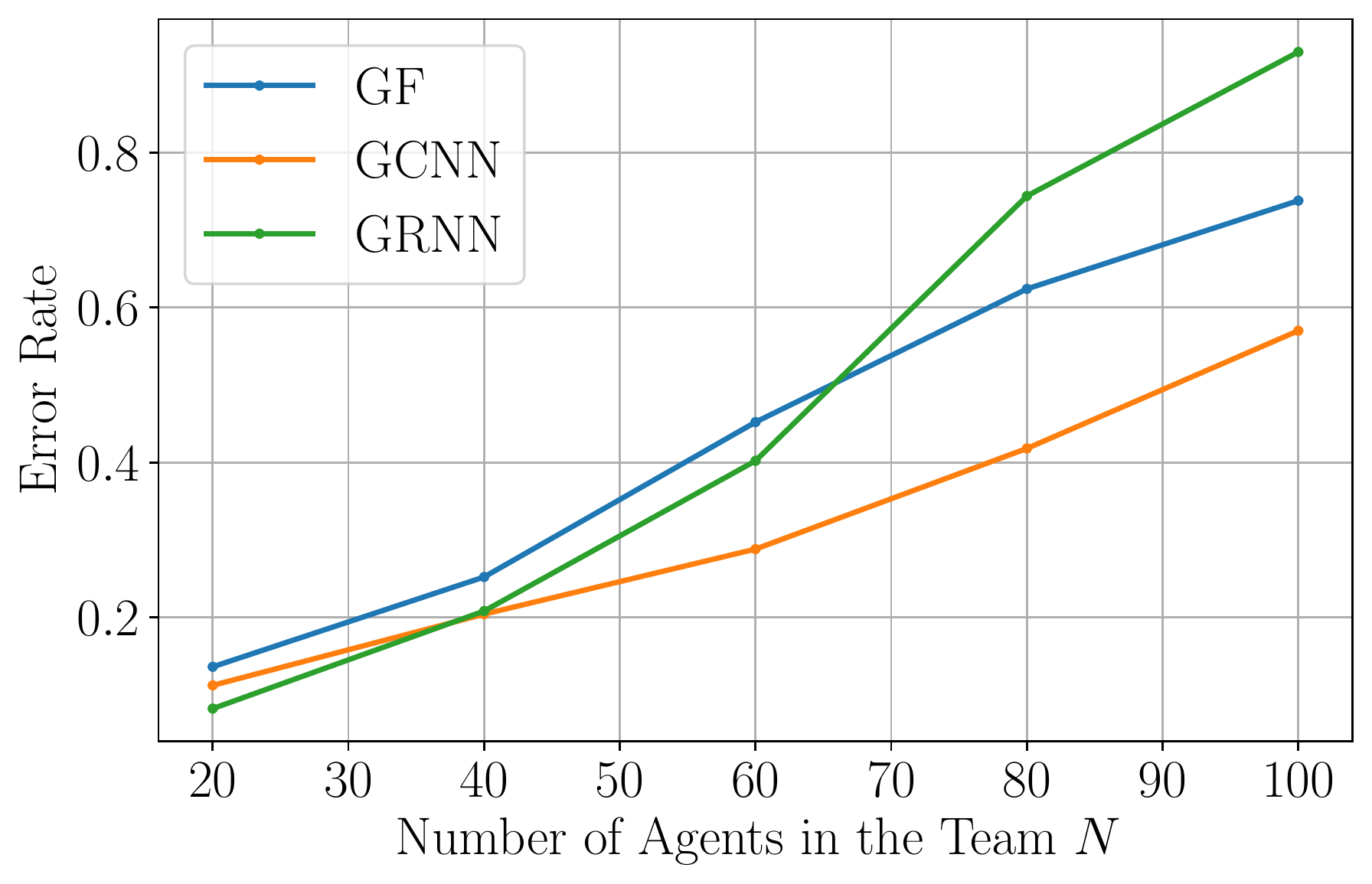}
        \caption{Transfer at scale -- Error rate}
        \label{subfig:SRtransferScale}
    \end{subfigure}
    \hfill
    \begin{subfigure}{0.3\textwidth}
        \centering
        \includegraphics[width = 0.9\textwidth]
        {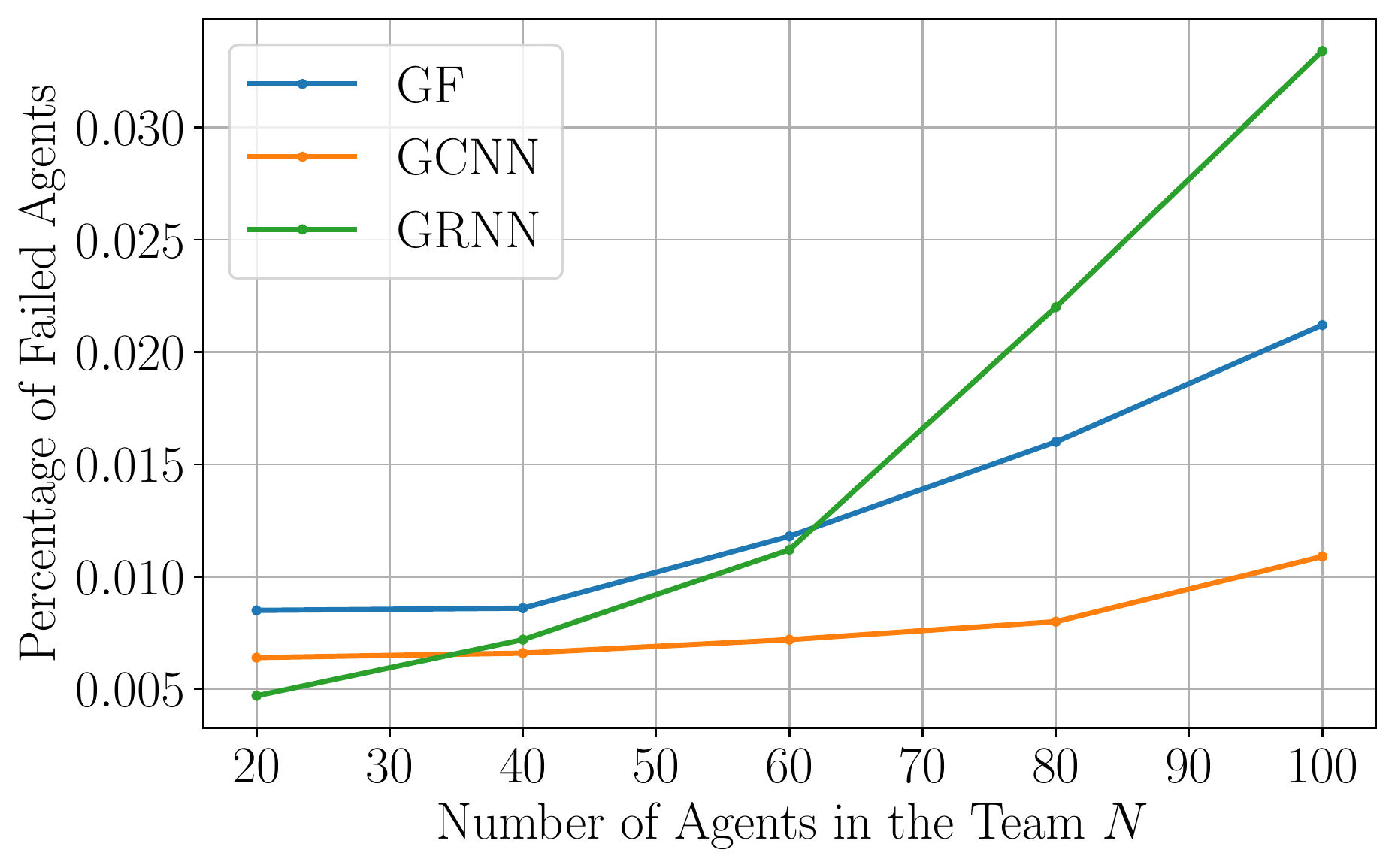}
        \caption{Transfer at scale -- Ratio of failed agents}
        \label{subfig:SAtransferScale}
    \end{subfigure}
    \hfill
    \begin{subfigure}{0.3\textwidth}
        \centering
        \includegraphics[width = 0.9\textwidth]
        {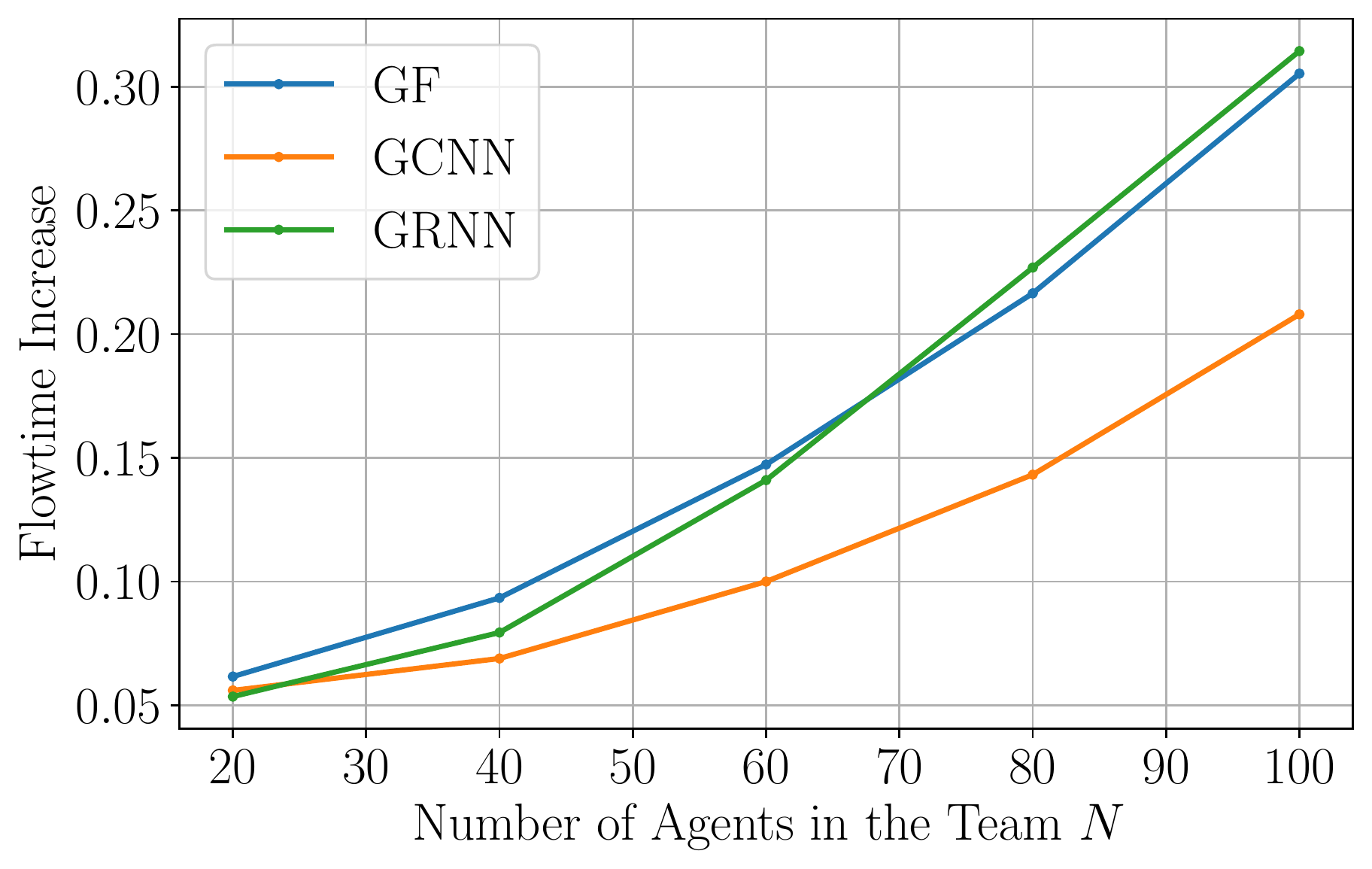}
        \caption{Transfer at scale -- Flowtime increase}
        \label{subfig:FTtransferScale}
    \end{subfigure}
    \hfill

    \hfill
    \begin{subfigure}{0.3\textwidth}
        \centering
        \includegraphics[width = 0.9\textwidth]
        {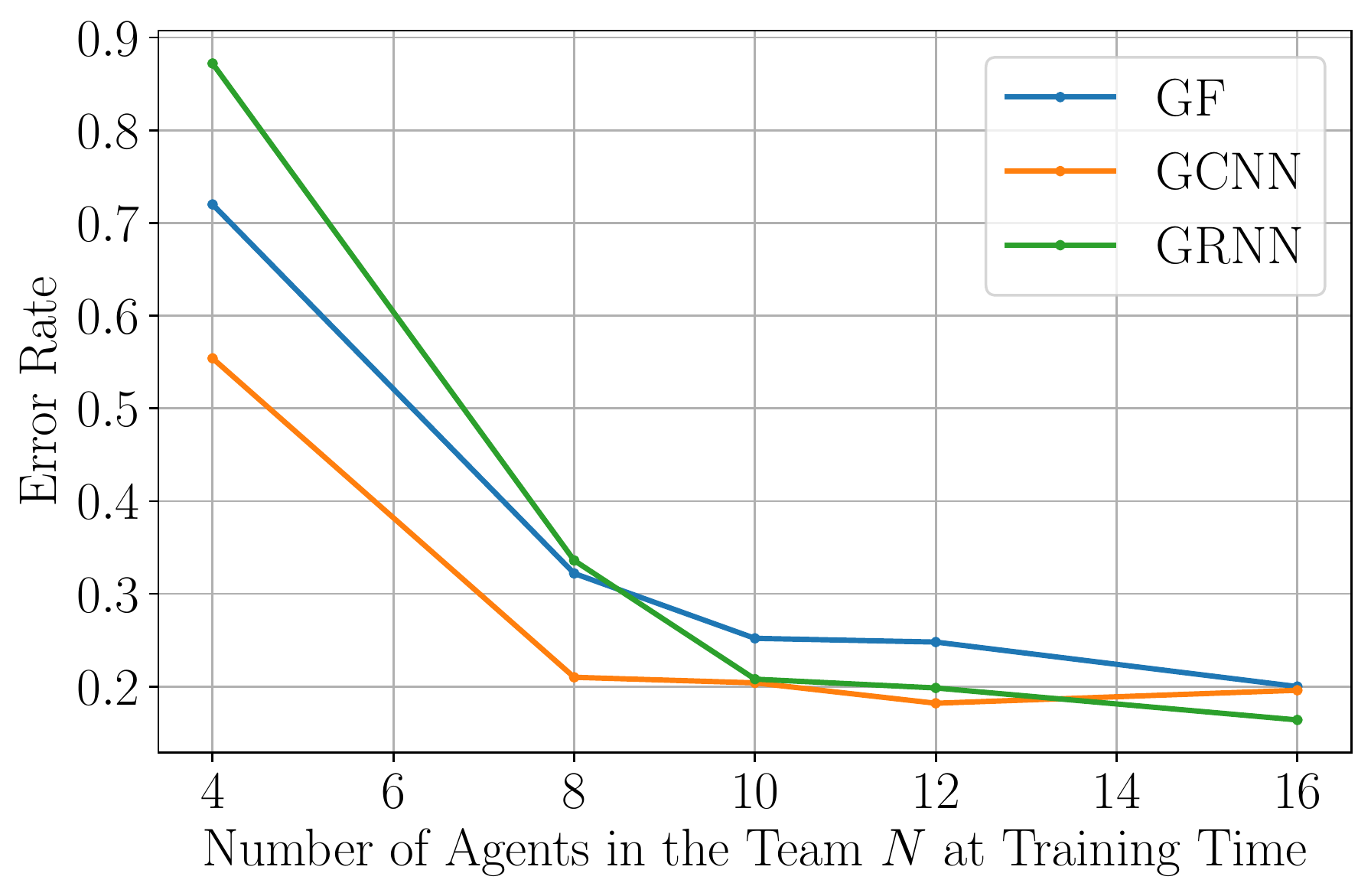}
        \caption{Threshold -- Error rate}
        \label{subfig:SRthreshold}
    \end{subfigure}
    \hfill
    \begin{subfigure}{0.3\textwidth}
        \centering
        \includegraphics[width = 0.9\textwidth]
        {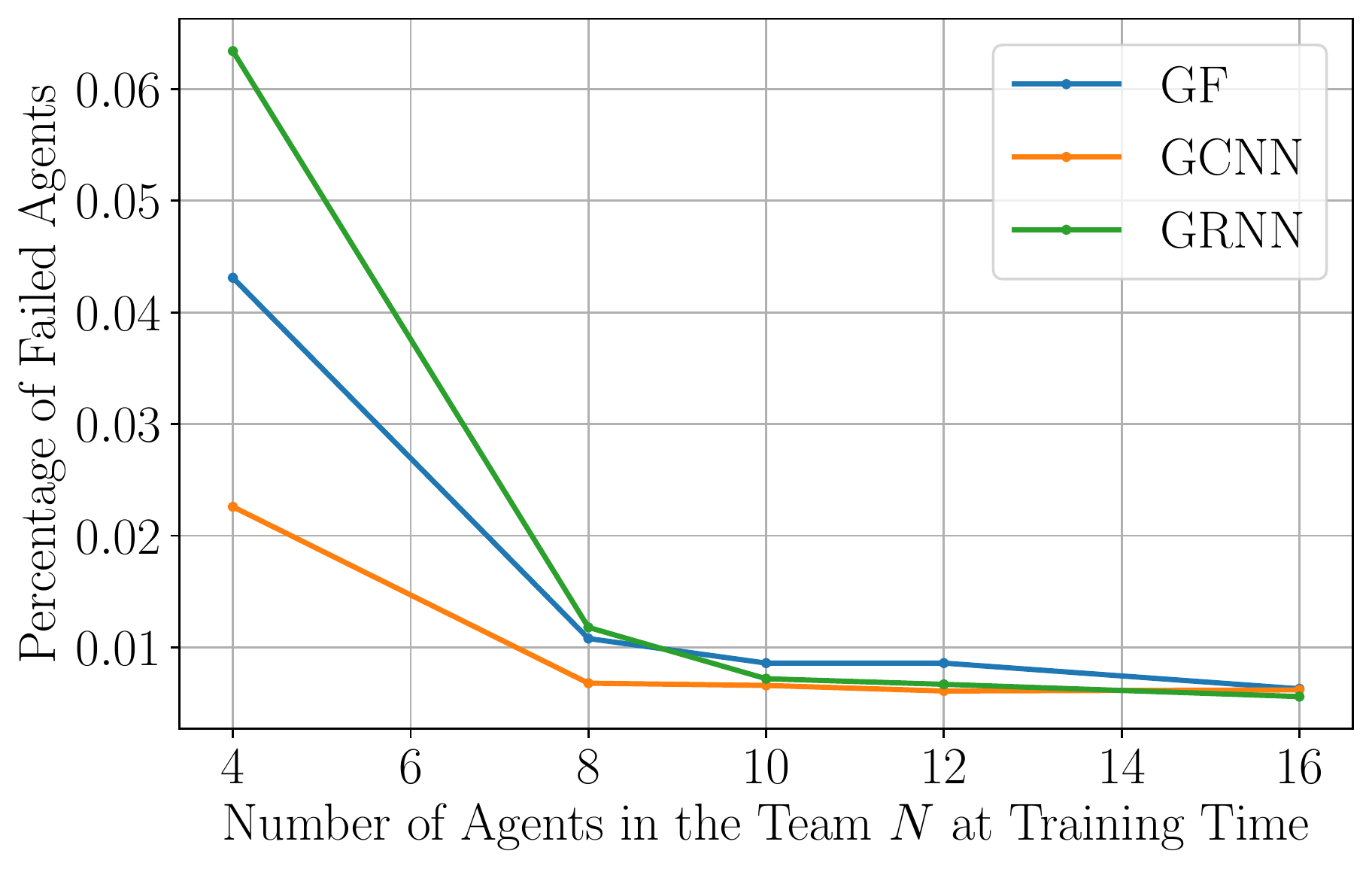}
        \caption{Threshold -- Ratio of failed agents}
        \label{subfig:SAthreshold}
    \end{subfigure}
    \hfill
    \begin{subfigure}{0.3\textwidth}
        \centering
        \includegraphics[width = 0.9\textwidth]
        {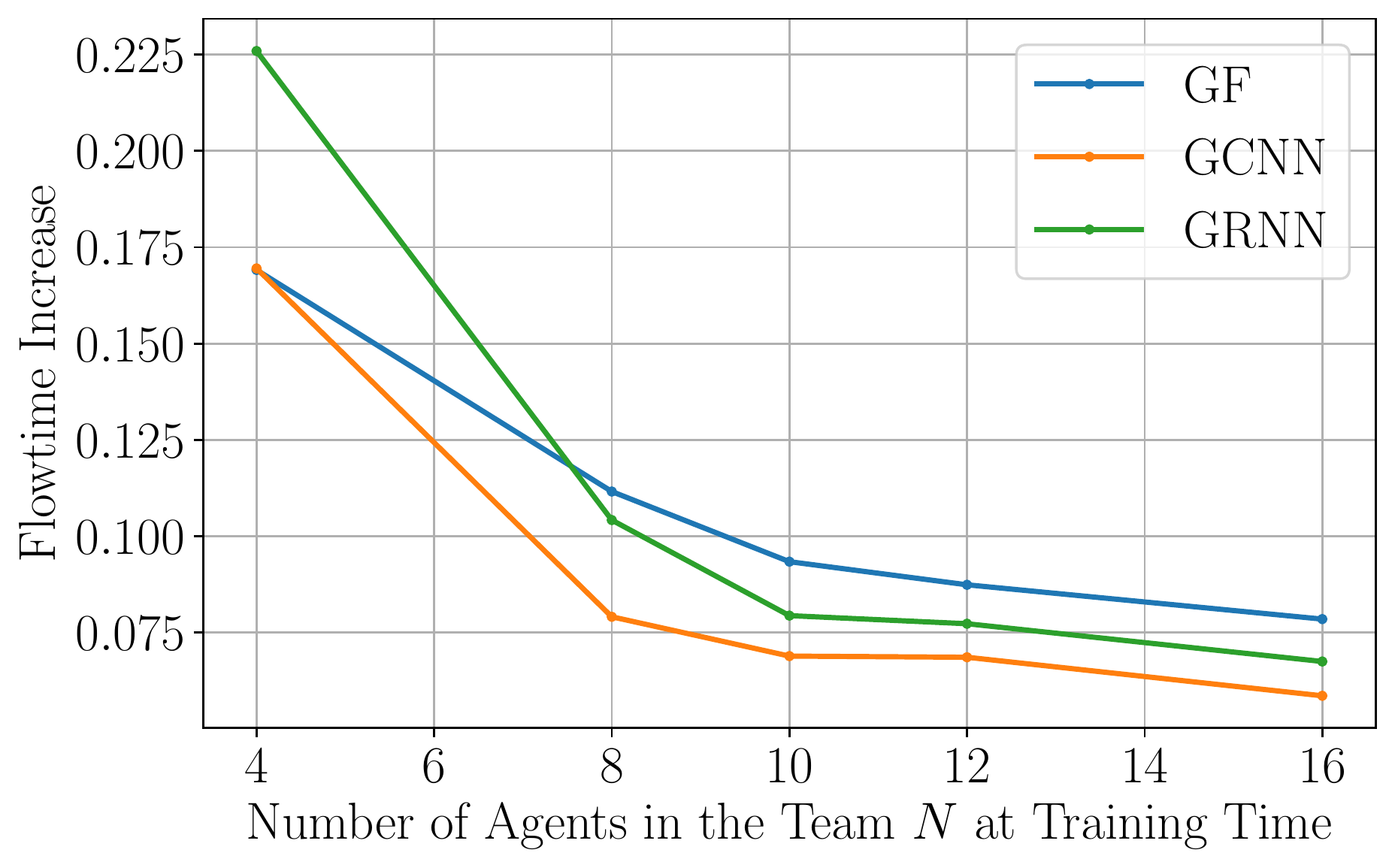}
        \caption{Threshold -- Flowtime increase}
        \label{subfig:FTthreshold}
    \end{subfigure}
    \hfill
    \caption{Experiments about transfering at scale. The learning architectures are trained on a $20 \times 20$ map with $10$ agents and tested on a $50 \times 50$ map with the number of agents indicated. \subref{subfig:SRtransferScale} Error rate: number of cases in the test set where \emph{any} agent fails to arrive at their goal. \subref{subfig:SAtransferScale} Ratio of failed agents (ratio of agents that fail to arrive at their goal). \subref{subfig:FTtransferScale} Increase in flowtime.}
    \label{fig:pathPlanningTransfer}
\end{figure*}


\myparagraph{Training.} We instantiate $600$ different maps, using $420$ for training, $90$ for validation, and $90$ for testing. For each map, we generate $50$ cases of placement of agents and targets. This generative process is done at uniformly at random. The models are trained for $150$ epochs with a batch size of $64$. We solve the imitation learning problem considering a cross-entropy loss function between the learned controller and the expert one, since we have only $5$ possible actions. We use the ADAM algorithm with learning rate that decays from $10^{-3}$ to $10^{-6}$ following a cosine annealing, and forgetting factors of $0.9$ and $0.999$, respectively. We also add a $L_{2}$ regularization to the trained parameters with a penalty of $10^{-5}$. During training, we consider that the trajectories are updated following an online expert~\cite{Li20-Planning}, whereby, on every validation step, we add $500$ new trajectories to the training set -- each trajectory obtained by using the controller learned up to that point, and then correcting them by means of the expert controller.

\subsection{Experimental Results}


\myparagraph{Experiment 1: Hyperparameter Choice.} First, we test different values of features $G \in \{16, 32, 64\}$ and filter taps $K \in \{2, 3, 4\}$. Results are summarized in Table.~\ref{tab:pathPlanning}. Similar to flocking, we see that the performance of the linear graph filter is worse than the nonlinear architectures based on the GCNN and GRNN, both in terms of error rate and flowtime increase. Unlike flocking, however, we see that the GCNN exhibits the best performance. This could be explained due to the fact that the communication clock and the decision clock are different. That is, since the GCNN is capable of capturing several hop information before the graph changes, it does not suffer from quick topology changes. Also, the nature of path planning is likely to rely more on current and future steps rather than previous steps (historical information). We also observe that the performance of the GRNN is highly sensitive to changes in the hyperparameter choice $G$ and $K$. From this simulation, we select the best pair $(G, K)$ for each of the three architectures and keep them for the following experiments; namely, $G=64, K=4$ for the graph filter, $G=64, K=3$ for the GCNN, and $G=64, K=2$, for the GRNN.


\myparagraph{Experiment 2: Initial conditions.} Second, we run tests for different initial conditions, namely different field of view size, different communication radius and different obstacle density in the environment. Results are shown in Fig.~\ref{fig:pathPlanningInit}. These experiments test the robustness of the architectures to a wider range of setups. First, we observe in Figs.~\ref{subfig:SRFOV}~and~\ref{subfig:FTFOV} that a bigger field of view leads to improved performance, both in terms of the error rate and the flowtime increase. This is because each agent has immediate access to more information about the environment and can thus improve its decisions based on this. We note that the GCNN seems to be more robust than the GRNN. When changing the communication radius, we observe in Fig.~\ref{subfig:SRradius} that an increased communication radius decreases the error rate slightly (for the GCNN and the GRNN) likely because more information is used to make a decision. However, as observed in Fig.~\ref{subfig:FTradius}, increasing the communication radius may cause redundant information exchanges, resulting in higher increase in the flowtime. Finally, we consider different obstacles densities, with results shown in Figs.~\ref{subfig:SRdensity}~and~\ref{subfig:FTdensity}. In this last case, as the obstacle density in the environment increases, the path planning becomes increasingly challenging, as evidence in the increase on the error rate (Fig.~\ref{subfig:SRdensity}) and the increase in the flowtime (Fig.~\ref{subfig:FTdensity}). Interestingly, the GRNN becomes better, relatively speaking, as the obstacle density increases. This could be explained by the fact that more obstacles implies more landmarks that can be learned (i.e., mapped through historical information) to assist navigation.


\myparagraph{Experiment 3: Transfer at scale.} Finally, we carried out the experiment on transferring at scale. Namely, we train the architectures on a $20 \times 20$ map with $10$ agents, and then we test it on a $50 \times 50$ map with an increasing number of agents $N \in \{20, 40, 60, 80, 100\}$, while maintaining the same obstacle density. We see the results in Fig.~\ref{fig:pathPlanningTransfer}. The error rate (Fig.~\ref{subfig:SRtransferScale}) increases considerably as the number of agents increases. Yet, it is still reasonable for $40$ agents (a team four times larger than at training time). However, recall that the error rate considers the number of cases where \emph{any} agent fails to reach their goal. When we, instead, look at the ratio of agents that failed (Fig.~\ref{subfig:SAtransferScale}), then we see that the architectures scale well, with less than $3\%$ of the agents failing to reach their goal when testing on a team with $100$ agents (ten times the size of the team at training time). With respect to the increase in flowtime (Fig.~\ref{subfig:FTtransferScale}), we observe a similar behavior, in that up to $40$ agents the learning architectures scale successfully. To complement these insights, we run a final experiment where we set the testing team size to $N=40$ agents, and we train on teams of increasingly larger size $N \in \{4, 8, 10, 12, 16\}$. The results in Figs.~\ref{subfig:SRthreshold}-\ref{subfig:FTthreshold} show a thresholding effect when transferring at scale:  There seems to be a minimum number of agents (at training time) that scales properly at test time, and after that number, adding more agents at training time does not necessarily improve performance. 


\begin{remark}\normalfont
The problem of path planning has been first addressed using a GCNN-based controller in \cite{Li20-Planning}. In this paper, we consider not only a new architecture (a GRNN) but also a whole new suite of experiments. First, it is observed that the GRNN-based controller does not improve the performance of the GCNN-based one. This suggests that the time variation in this problem is not significant (see discussion in Experiment~1). Second, we have carried out new experiments for different scenarios, varying the size of the field of view, the communication radius, and the obstacle density (Fig.~\ref{fig:pathPlanningInit}) to illustrate how the GNN-based controllers perform under different practical settings. Third, and most importantly, we have carefully analyzed the transferability at scale, uncovering a thresholding effect. Specifically, Figs.~\ref{subfig:SRthreshold},~\ref{subfig:SAthreshold},~and~\ref{subfig:FTthreshold} show that there is a minimum number of agents required at training time for the transferability at scale to hold. Once this minimum number of agents is met, the performance does not significantly improve when initially trained on a larger number of agents, exhibiting a thresholding effect.
\end{remark}


\section{Conclusions} \label{sec:conclusions}



The problem of controlling dynamical systems comprised of multiple autonomous agents, resides in the difficulty to find optimal controllers that respect the decentralized nature of the system. We postulate the use of graph neural networks as appropriate parametrizations for such controllers. Then, the problem reduces to finding the optimal set of parameters, finding the best possible controller within the space of graph neural networks, which can be achieved by leveraging the imitation learning framework.

Graph neural networks, in particular, graph convolutional neural networks and graph recurrent neural networks, exhibit several desirable properties in the context of decentralized control. Most importantly, they are naturally distributed and local operations, thus adapting seamlessly to the information structure imposed by the decentralized dynamical system. They are also capable of learning nonlinear behaviors which, knowing that most optimal controllers are nonlinear, becomes a relevant feat. Furthermore, both architectures are permutation equivariant and stable to changes in the graph. We then proved that, under an appropriate choice of description for the observations, this implies that the problem is permutation invariant, meaning that the same learned parameters are useful in a wide range of similar problems. Together, these properties warrant the scalability and transferability of the learned controllers.

We have tested the learned controllers in two proof-of-concept problems, namely, flocking and path planning. In both cases, we observed that the graph neural network controllers exhibit better performance than linear, decentralized controllers. More importantly, we showed the capabilities of these controllers to be deployed in larger environments. In the problem of flocking, controllers trained on a small team of $50$ agents work perfectly well in teams of at least $100$ agents, doubling the size; while in the problem of path planning, controllers trained in as little as $10$ agents are successful in teams of $40$ agents, adapting to a fourfold increase.

Overall, the proposed framework has shown promising results in simple, but useful settings. Moving forward, we expect to test graph neural networks as parametrizations in more complex control problems, investigating the impact and limitations of this choice. We further expect to develop more useful properties of these architectures as they pertain to the decentralized control problem, such as robustness, safety and resilience. Future research directions include studying the closed-loop stability of GNN-based controllers learned by means of imitation learning, leveraging the knowledge of the dynamical system to improve on the controllers and exploring alternative learning frameworks.



\appendices

\section{Permutation Invariance} \label{sec:appendix}



\begin{proof}[Proof of Proposition~\ref{thm:permutationInvariance}]
    To prove \eqref{eq:permutationInvariance} we need to prove that the objective functions in \eqref{eq:imitationLearning} and \eqref{eq:imitationLearning_tilde} are equivalent. Let us start with the objective function in \eqref{eq:imitationLearning_tilde}
    \begin{equation} \label{eq:objFunctionPermutation}
    \Big\| \sfPhi \big(\bbP^{\Tr} \bbX(t); \bbP^{\Tr} \bbS \bbP, \ccalH \big) - \Pi^{\ast}_{c}(\tbX(t)) \Big\|
    \end{equation}
    where we have replaced $\tbX$ and $\tbS$ by the corresponding definitions. We know that $\sfPhi$ is permutation equivariant as long as it is a graph filter \cite[Prop.~1]{Gama20-Stability}, a GCNN \cite[Prop.~2]{Gama20-Stability} or a GRNN \cite[Prop.~1]{Ruiz20-GRNN}, which means that
    \begin{equation} \label{eq:permutationEquiv}
    \sfPhi(\bbP^{\Tr} \bbX; \bbP^{\Tr} \bbS \bbP, \ccalH) = \bbP^{\Tr} \sfPhi(\bbX;\bbS, \ccalH).
    \end{equation}
    We also know that, since the cost $c$ satisfies Def.~\ref{def_sys_invariance}, then $\Pi_{c}^{\ast}(\tbX(t)) = \bbP^{\Tr}\Pi_{c}^{\ast}(\bbX(t))$ [cf. \eqref{eq:generalObjective}]. Using these two facts in \eqref{eq:objFunctionPermutation} we get
    \begin{equation} \label{eq:objFunctionPermutationOut}
    \begin{aligned}
    \Big\| & \bbP^{\Tr} \sfPhi \big( \bbX(t);  \bbS , \ccalH \big) - \bbP^{\Tr} \Pi_{c}^{\ast}(\bbX(t)) \Big\| \\ & = \Big\| \bbP^{\Tr} \Big( \sfPhi \big( \bbX(t);  \bbS , \ccalH \big) - \Pi_{c}^{\ast}(\bbX(t)) \Big) \Big\|.
    \end{aligned}
    \end{equation}
    Recall that for a graph signal $\bbX \in \reals^{N \times F}$ we have $\|\bbX\| = \sum_{f=1}^{F} \| \bbx_{f}\|_{2}$, so that
    \begin{equation} \label{eq:objFunctionPermutationEquiv}
    \begin{aligned}
    \Big\| \bbP^{\Tr} \Big( \sfPhi & \big( \bbX(t);  \bbS , \ccalH \big) - \Pi_{c}^{\ast}(\bbX(t)) \Big) \Big\| \\ & = \Big\|  \sfPhi \big( \bbX(t);  \bbS , \ccalH \big) - \Pi_{c}^{\ast}(\bbX(t)) \Big\|.
    \end{aligned}
    \end{equation}
    since $\bbP$ is a permutation matrix. Following \eqref{eq:objFunctionPermutation}, \eqref{eq:objFunctionPermutationOut} and \eqref{eq:objFunctionPermutationEquiv}, we get that \eqref{eq:imitationLearning_tilde} yields the same objective function as \eqref{eq:imitationLearning}, and that solving one or the other yields the same set of filter taps, hereby completing the proof.
\end{proof}


\bibliographystyle{bibFiles/IEEEtranD}
\bibliography{bibFiles/myIEEEabrv,bibFiles/biblioControlGNN}

\end{document}